\documentclass[11pt]{article}
\usepackage[margin=1in]{geometry}

\usepackage{amssymb}
\usepackage{latexsym}
\usepackage{amsmath}
\usepackage{amsthm}
\usepackage{thmtools}
\usepackage{thm-restate}
\usepackage{xcolor}
\usepackage[round]{natbib}
\usepackage{nameref}
\usepackage[colorlinks, linkcolor=blue, filecolor = blue, citecolor = blue, urlcolor = magenta]{hyperref}
\usepackage[noabbrev,capitalize]{cleveref}
\usepackage{graphicx}

\usepackage[extdef=true]{delimset}
\usepackage{enumitem}
\usepackage{algorithm}
\usepackage{algorithmic}
\usepackage{nicefrac}

\allowdisplaybreaks
\newcommand{\papertitle}{{Faster Convergence with Multiway Preferences}}

\newtheorem{thm}{Theorem}
\newtheorem{lem}[thm]{Lemma}

\newtheorem{rem}{Remark}

\newcommand{\bign}[1]{\big(#1\big)}
\newcommand{\biggn}[1]{\bigg(#1\bigg)}
\newcommand{\Bign}[1]{\Big(#1\Big)}

\newcommand{\R}{{\mathbb R}}

\newcommand{\N}{{\mathbb N}}
\renewcommand{\P}{{\mathbf P}}

\newcommand{\E}{{\mathbf E}}
\newcommand{\I}{{\mathbf I}}
\newcommand{\1}{{\mathbf 1}}
\newcommand{\0}{{\mathbf 0}}

\newcommand{\cB}{{\mathcal B}}

\newcommand{\cS}{{\mathcal S}}
\newcommand{\cI}{{\mathcal I}}

\newcommand{\cG}{{\mathcal G}}

\newcommand{\cN}{{\mathcal N}}

\newcommand{\cD}{{\mathcal D}}

\newcommand{\cR}{{\mathcal R}}
\newcommand{\cH}{{\mathcal H}}
\renewcommand{\S}{{\mathcal S}}

\newcommand{\tw}{\tilde \w}

\renewcommand{\b}{{\mathbf b}}
\renewcommand{\o}{{\mathbf o}}
\newcommand{\m}{\mathbf m}

\newcommand{\g}{{\mathbf g}}

\newcommand{\n}{{\mathbf n}}

\renewcommand{\u}{{\mathbf u}}
\renewcommand{\v}{{\mathbf v}}
\newcommand{\w}{{\mathbf w}}
\newcommand{\x}{{\mathbf x}}
\newcommand{\y}{{\mathbf y}}
\newcommand{\z}{{\mathbf z}}

\renewcommand{\ref}{\cref}

\newcommand{\probsgn}{\emph{Sign-Feedback-Optimization}}
\newcommand{\probll}{\emph{Batched Sign-Feedback-Optimization}}

\newcommand{\probarg}{\emph{Battling-Feedback-Optimization}}

\newcommand{\nabf}{{\nabla f}}

\newcommand{\sign}{\text{sign}}
\newcommand{\dotp}{\boldsymbol\cdot}
\newcommand{\tr}{{^\top}}

\newcommand{\red}[1]{\textcolor{red}{\bf {#1}}}

\title{\papertitle}

\author{
Aadirupa Saha%
\thanks{Apple; {\tt aadirupa@apple.com}.}
\and 
Vitaly Feldman
\thanks{Apple; {\tt vitalyf@apple.com}.}
\and
Tomer Koren%
\thanks{Tel Aviv University and Google Tel Aviv; {\tt tkoren@tauex.tau.ac.il}.} 
\and Yishay Mansour%
\thanks{Tel Aviv University and Google Tel Aviv; {\tt mansour.yishay@gmail.com}.}
}
\date{}

\begin{document}

\maketitle


\begin{abstract}
We address the problem of convex optimization with preference feedback, where the goal is to minimize a convex function given a weaker form of comparison queries. Each query consists of two points and the dueling feedback returns a (noisy) single-bit binary comparison of the function values of the two queried points. Here we consider the sign-function-based comparison feedback model and analyze the convergence rates with batched and multiway (argmin of a set queried points) comparisons. Our main goal is to understand the improved convergence rates owing to parallelization in sign-feedback-based optimization problems. Our work is the first to study the problem of convex optimization with multiway preferences and analyze the optimal convergence rates.
Our first contribution lies in designing efficient algorithms with a convergence rate of $\smash{\widetilde O}(\frac{d}{\min\{m,d\} \epsilon})$ for $m$-batched preference feedback where the learner can query $m$-pairs in parallel. We next study a $m$-multiway comparison (`battling') feedback, where the learner can get to see the argmin feedback of $m$-subset of queried points and show a convergence rate of $\smash{\widetilde O}(\frac{d}{  \min\{\log m,d\}\epsilon })$. We show further improved convergence rates with an additional assumption of strong convexity. Finally, we also study the convergence lower bounds for batched preferences and multiway feedback optimization showing the optimality of our convergence rates w.r.t. $m$.
\end{abstract}


\vspace{-13pt}
\section{Introduction}
\label{sec:intro}
\vspace{-5pt}

Studying the problem of convex optimization presents a unique opportunity to delve deep into a practical field of vast applications and make a lasting impact in both academia and industry.
Most commonly, convex optimization is studied in a first-order gradient oracle model, where the optimization algorithm may query gradients of the objective function; or a more limited model of zero-order oracle access, where the optimization algorithm may only query function values. Such optimization frameworks are well-studied in the literature~\citep[see, e.g.,][]{nesterovbook,hazanbook,bubeckbook}.

One major limitation of the above optimization frameworks lies in the feedback model: In many practical applications, obtaining complete gradient information or even access to a function value (zeroth-order) oracle could be difficult. E.g. in recommender systems, online shopping, search engine data, the only data available to the learning algorithm specifies only the preference feedback of their users across multiple choices. Similar problems may arise in other real-world domains including crowd-sourcing surveys, drug testing, tournament ranking, social surveys, etc.  

\textbf{Learning with Preference Feedback.} The above line of problems gave rise to a new field of research where the system aims to optimize its performance based on only choice data or relative preferences across multiple items. The problem has been widely studied in the bandit community as Dueling Bandit problems which is an online sequential game where at each round the learner (algorithm) repeatedly selects a pair of items to be
compared to each other in a ``duel,'' and consequently observes a binary stochastic preference feedback of the winning item in this duel
\citep{Ailon+14,DTS,sui2018advancements}. The goal of the learner is to identify the `best item' with the least possible number of pairwise queries. 

\textbf{Need of Optimization Algorithms with Preference Feedback. } The classical problem of dueling bandits, although received wide attention in the learning theory community, most of the studies were limited to finite decision space, which allowed the learner to examine each item one-by-one before identifying the best one \cite{Busa21survey}. This again becomes unrealistic for large-scale real-world problems where decision spaces could be extremely large. Precisely, the literature lacks optimization methods based on preference feedback, where the decision set could be potentially infinite. 

\textbf{Related Works and Limitations. } Two earlier works that address the problem of convex optimization with preference feedback are \cite{Yue+09,Jamieson12}. However, the first work yields suboptimal convergence bounds and \cite{Jamieson12} only deals with strongly convex and smooth functions.   
Another recent work by \cite{SKM21} addresses the problem of convex optimization with pairwise preference setting which is closest to our framework: The setting assumes an underlying (convex) loss function $f:\R^d \mapsto R$ and at each round the learner can see the relative ordering of the function values at the two queried points: Precisely, upon querying a duel $(\x_t,\y_t)$ at round $t$, the learner can observe $\sign(f(\x_t) -f(\y_t))$, and the objective of the learner is to find a `near-minimizer' of $f$ as fast as possible. 
Despite their setup having an interesting angle of optimization with preference feedback, their feedback model is limited to only pairwise/ dueling queries. 

\vspace{-1pt}
\textbf{Motivation of Our Work: Optimization with Multiway Preferences. } While pairwise comparison feedback is perhaps the simplest to model and analyze, in most real-world systems, users get to make a choice from a set of options, be that in online marketplaces, Youtube recommendations, Google maps, restaurant selection and many more. This raises a natural question about the relative power of multiway comparison feedback in terms of the query complexity. Further, in many settings, it is not feasible to update the model's predictions after every comparison feedback provided by the user, for example, due to communication delays. Instead, the system can ask a number of comparison queries in parallel and then update its state (and generate the next set of queries). In such settings, it is natural to ask how many such rounds of $m$ queries would be necessary to identify the (approximate) minimizer?

%
%
\vspace{-1pt}

\textbf{Some Negative Results on Multiway Preferences in Bandits Literature: }
The setting of multiway preferences was studied as a generalization of the dueling bandit framework, however for finite decision spaces and for a very type of Multinomial Logit (or Plackett Luce) based preference model \cite{SGwin18,SGrank18,Ren+18}. However, their specific feedback model was not able to exploit the power of multiway queries, precisely, they show multiway feedback may not yield faster convergence results, even for finite decision space settings (when $\cD$ is finite). 

But in this work, we answer the above questions in the affirmative and studied two specific types of multiway preference models which can indeed yield faster convergence rates with larger strength of multiway queries $(m)$. \emph{It is important to note that our results do not contradict the negative results with Multinomial Logit (MNL) models \citep{SGrank18,ChenSoda+17,Ren+18,SGpac20} as we use a different `argmin' based preference model as opposed to the MNL model. One of our main strengths lies in identifying such a subsetwise preference feedback model which could exploit the strength of multiway preferences.} The noisy-winner feedback in the MNL model increases the variability which nullifies the strengths of querying larger subsets \cite{SGwin18,ChenSoda+18}, but our proposed algorithms show how to exploit the latter with our multiway preference models (\cref{sec:batch,sec:argmin}). To the best of our knowledge, our work is the first to study the problem of convex optimization (on infinite decision space) with multiway preferences and analyze the optimal convergence rates. 



\vspace{-2pt}
\paragraph{Our contributions.}  The specific contributions of our are listed below:
\begin{enumerate}[leftmargin=0pt, itemindent=10pt, labelwidth=5pt, labelsep=5pt,itemsep=2pt,topsep=0pt]
	\item Our first contribution is to propose two multiway preference feedback models for optimizing fixed convex functions $f:\cD \mapsto \R$: (1) \probll: In this setting, the learner can get to query a subset of $m \leq d$ distinct pair of points and receives the $\sign$ (or comparison) feedback of each queried pairs. This can be seen as a batched preference feedback model where the learner can simultaneously query $m$ pairs of duels (2) \probarg: Unlike the previous model, in this case, the learner gets to query a subset of $m$ points in the decision space $\cD$ and only gets to see the minimizer of the queried set of $m$-points. We called this as \emph{Battling (or Multiwise Winner)} feedback model, inspired from \cite{SG18} (see \cref{sec:problem} for details).

	\item We first consider the batched feedback model. Assuming $f$ is $\beta$-smooth we apply an `aggregated normalized gradient descent' based routine that is shown to yield $ O\Big(\frac{d\beta D}{\epsilon \min\{m,d\}} \Big)$ convergence rate for finding an $\epsilon$-optimal point in $\cD$ (\cref{alg:batch}, \cref{thm:batch}). Following this we also propose an `epochwise warm start with smooth optimization blackbox' idea to yield a faster convergence rate of $O\Big( \frac{d \beta}{\min\{m,d\}\alpha}\log_2\big(\frac{\alpha}{\epsilon}\big) \Big)$ with an additional assumption of strong convexity on $f$ (\cref{alg:batch_strng}, \cref{thm:batch_strng}). We also suggest how to deal with noisy preference feedback for these settings in \cref{rem:batch_noise} (\cref{sec:batch}).
	
	\item In \cref{sec:argmin}, we propose optimization algorithms for the \probarg \, (with Multiwise-winner) problem. In this case, we first design a novel convergence routine (\cref{alg:argmin}) that yields $ O\Big(\frac{d\beta D}{\min\{\log m,d\}\epsilon } \Big)$ convergence rate for the class of smooth convex functions (\cref{thm:argmin}). The key novelty in the algorithm lies in querying structured $m$ subsets at each round, exploiting which we show how to extract $\log m$ distinct pairwise sign feedback and again use an aggregated normalized gradient descent method to yield the desired convergence bound. Following this we also show a faster $ O\Big( \frac{d \beta}{\min\{\log m,d\}\alpha }\log_2\big(\frac{\alpha}{\epsilon}\big)\Big)$ convergence rate with strong convexity in \cref{thm:argmin_strng}. We also remark on how to deal with noisy feedback models in this case as well.
 
	\item Finally we show matching convergence lower bounds for both the multiway preference feedback models (resp. in \cref{thm:lb_batch} \cref{thm:lb_argmin}), which shows our dependencies on the multiway parameter $m$ are indeed optimal, i.e. our algorithms are able to exploit the strength of multiway queries $(m)$ optimally. 
        \item We provide empirical evaluations to corroborate our theoretical findings in \cref{sec:expts}.
	
	\item Finally another minor contribution lies in dealing with bounded decision space throughout, unlike \cite{SKM21} which assumes the decision space to be unbounded (\cref{sec:single}).
\end{enumerate}


\section{Preliminaries and Problem Statement}
\label{sec:problem}
\vspace{-6pt}

\textbf{Notation.}  Let $[n] = \{1, \ldots n\}$, for any $n \in \N$. 
Given a set $S$ and two items $x,y \in S$, we denote by $x \succ y$ the event $x$ is preferred over $y$. 
For any $r > 0$, let $\cB_d(r)$ and $\cS_d(r)$ denote the ball and the surface of the sphere of radius $r$ in $d$ dimensions respectively.
$\I_d$ denotes the $d \times d$ identity matrix. 
For any vector $\x \in \R^d$, $\|\x\|_2$ denotes the $\ell_2$ norm of vector $\x$. 
$\1(\varphi)$ is generically used to denote an indicator variable that takes the value $1$ if the predicate $\varphi$ is true and $0$ otherwise. 
$\sign(x) = +1$ if $x \ge 0$ or $-1$ otherwise, $\forall x \in \R$. $\text{Unif}(S)$ denotes a uniform distribution over any set $S$. 
 We write $\tilde{O}$ for the big O notation up to logarithmic factors.


\subsection{Problem Setup.} 
We address the convex optimization problem with binary-valued $\sign$ preference feedback: Assume $f: \cD \mapsto \R$ be any convex map defined on a convex set $\cD \subseteq \R^d$. At every iteration $t$, the goal of the learner is to pick a pair of points $(\x_t,\y_t)$, upon which it gets to see a binary $0-1$ bit noisy comparison feedback $o_t$ s.t.:
\vspace{-5pt}
\[
Pr[o_t = \sign(f(\x_t) - f(\y_t))] = 1 - \nu,
\] 
\vspace{-1pt}
where $\nu \in (0,1/2]$ is the (unknown) noise-parameter, $\nu = 0$ corresponds to pure $\sign$ feeedback (without any noise).

We consider the following two generalizations of $\sign$-feedback considered in \cite{SKM21}.

\vspace{2pt}
\textbf{Batched-Sign Feedback: } In this setting, at any round $t$, the learner can query $m$-parallel (batched) pair of points $\{(\x_t^i,\y_t^i)\}_{i = 1}^m$ and gets to see the $\sign$ feedback for each pair of points, i.e. the learner receive $m$-bits of $\sign$ feedback $\{o_t^i\}_{i = 1}^m$ such that 
$
Pr[o_t^i = \sign(f(\x_t^1) - f(\y_t^i))] = 1 - \nu, ~ i \in [m]. 
$

\vspace{2pt}
\textbf{Battling (Multiwise-Winner) Feedback: } In this setting, at any round $t$, the learner can query a set $S_t$ of $m$ points $S_t = (\x_t^1,\x_t^2,\ldots,\x_t^m)$ and gets to see the $\arg\min$ feedback of the $m$-points -- i.e., the learner receive only $1$-bit of $\arg\min$ feedback $\{o_t \in [m]\}$ such that: 
$
Pr[o_t = \arg\min(f(\x_t^1), f(\x_t^2), \ldots, f(\x_t^m))] = 1 - \nu. 
$

\vspace{3pt}
\textbf{Objective.} We consider the objective of minimizing the function sub-optimality gap: So if $\x_{T+1}$ is the point suggested by the algorithm after $T$ rounds, then the goal is to 
\[
\min_{\x \in \cD} \bign{\E[f(\x_{T+1})] - f(\x^*)},
\]
with least number of queries $(T)$ possible.

\section{Batched Dueling Convex Optimization with Sign Feedback}
\label{sec:batch}

We first analyze the \probll\, problem, where at each iteration the learner can query $m$-parallel (batched) pair of points $\{(\x_t^i,\y_t^i)\}_{i = 1}^m$ and gets to see the $\sign$ feedback for each pair of points, $\{o_t^i\}_{i = 1}^m$ such that 
$
Pr[o_t^i = \sign(f(\x_t^1) - f(\y_t^i))] = 1 - \nu, ~ i \in [m], ~(m \leq d). 
$
We present two algorithms for this setup, respectively for smooth and strongly convex optimization settings and show an $O(1/m)$ improved convergence rate in both settings compared to the single pair $\sign$-feedback setting addressed in \citep{SK21} (see \cref{thm:batch} and \cref{thm:batch_strng}). 
We analyze the above algorithms for the noiseless setting (i.e. $\nu = 0$), but \cref{rem:batch_noise} discusses how they can be easily extended to the noisy $\sign$-feedback setup for any arbitrary $\nu \in [0,0.5)$. 
Following this we also prove a convergence lower bound for the batched feedback setting which shows our $1/m$ rate of improvement with the batch size $m$ is indeed optimal (see \cref{thm:lb_batch}).

\subsection{Proposed Algorithm: Batched-NGD} 
\label{sec:alg_batch}
The main idea in this setup is to estimate gradient directions (normalized gradient estimates) in $m$ different directions and take an aggregated descent step. Formally, at each round $t$, we can query $m$ iid random unit directions, say $\u_t^1,\ldots,\u_t^m \overset{\text{iid}}{\sim} \text{Unif}(\cS_d(1))$, and find the normalized gradient estimates $g_t^i = o_t^i\u_t^i$ along each direction, where $o_t^i = \sign\bign{f(\x_t^i) - f(\y_t^i)}$ is the sign feedback of the $i$-th pair of queried duel $(x_t^i,y_t^i)$. Subsequently, we update the running prediction as $\w_{t+1} \leftarrow \w_t - \eta \g_t$, where $\g_t = \frac{1}{m}\sum_{i = 1}^m\g_t^i$ denotes the aggregated normalized gradient estimate. The algorithm also maintains a running minimum $\m_t$ which essentially keeps track of $\min\{\w_1,\ldots,\w_t\}$\footnote{Interested readers can check the analysis of Projected-Normalized Gradient Descent algorithms for single $\sign$-Feedback in \cref{sec:single}. This is unlike the version studied in the literature which considered unconstrained optimization (i.e. $\cD = \R^d$) \citep{SKM21}, although the analysis is quite similar, except we have to account for the projection step additionally. This also leads to simpler tuning of the perturbation parameter $\gamma$ in our case.}. The complete algorithm is given in \cref{alg:batch}.

\vspace{-10pt}
\begin{center}
	\begin{algorithm}[h]
		\caption{\textbf{Batched-NGD (B-NGD)}} 
		\label{alg:batch}
		\begin{algorithmic}[1] 
			\STATE {\bfseries Input:} Initial point: $\w_1 \in \cD$, Learning rate $\eta$, Perturbation parameter $ \gamma$, Query budget $T$ (depends on error tolerance $\epsilon$), Batch-size $m$ 
                \STATE {\bf Initialize} Current minimum $\m_1 = \w_1$
			\FOR{$t = 1,2,3,\ldots, T$}
			\STATE Sample $\u_t^1,\u_t^2,\ldots\u_t^m \sim \text{Unif}(\cS_d(1))$ 
			\STATE Set  $\x_{t}^{i} := \w_t +  \gamma \u_t^i$,~ 
			            $\y_{t}^{i} := \w_t -  \gamma \u_t^i$
			\STATE Play the duel $(\x_{t}^{i},\y_t^{i})$, and observe $o_t^i \in {\pm 1}$ such that $o_t^i = \sign\big( f(\x_{t}^{i}) - f(\y_{t}^{i}) \big)$. 
			\STATE Update $\tw_{t+1} \leftarrow \w_t - \eta \g_t$, 
			where $\g_t = \frac{1}{m}\sum_{i=1}^m\g_t^i$, $\g_t^i = o_t^i \u_t^i$
			\STATE Project $\w_{t+1} = \arg\min_{\w \in \cD}\norm{\w - \tw_{t+1}}$
			\STATE Query the pair $(\m_{t}, \w_{t+1})$ and receive $\sign\bign{ f(\m_{t}) - f(\w_{t+1}) }$. 
			\STATE Update $\m_{t+1} \leftarrow 
			\begin{cases}
				\m_{t} ~\text{ if } \sign\bign{ f(\m_{t}) - f(\w_{t+1}) }<0\\
				\w_{t+1} ~\text{ otherwise }  
			\end{cases}$
			\ENDFOR   
			\STATE Return $\m_{T+1}$
		\end{algorithmic}
	\end{algorithm}
\end{center}
\vspace{-10pt}

\begin{restatable}[Convergence Analysis of \cref{alg:batch} for $\beta$-Smooth Functions]{thm}{thmbatch}
\label{thm:batch}
Consider $f$ to be $\beta$ smooth. Suppose \cref{alg:batch} is run with $\eta = \frac{m\sqrt{\epsilon}}{20\sqrt{ d \beta}}$,  $\gamma = \frac{ \epsilon^{3/2}}{960 \beta d\sqrt d D^2\ \sqrt{\log 480}}\sqrt{\frac{2}{\beta}}$ and $T_\epsilon =  O\Big(\frac{d\beta D}{m \epsilon} \Big)$, where $\|\w_{1} - \x^*\|^2 \leq D$ (is an assumed  known upper bound). Then \cref{alg:batch} returns $\E[f(\tw_{T+1})] - f(\x^*) \le \epsilon$ with sample complexity $2T_\epsilon$, for any $m \leq d$. 
\end{restatable}

\begin{proof}[Proof Sketch of \cref{thm:batch}]
We start by noting that by definition: 
$
\norm{\w_{t+1} - \x^*}^2 \leq \norm{\tilde \w_{t+1} - \x^*}^2  =  \norm{\w_{t} -\frac{\eta}{m}\sum_{i = 1}^m\g_t^i - \x^*}^2,
$
where the first inequality holds since projection reduces the distance to optimal $\x^*$. This further leads to
\begin{align*}
& {m^2}\norm{\w_{t+1} - \x^*}^2  
 = m(\norm{\w_{t}-\x^*}^2 + \eta^2)  -2\eta\sum_{i = 1}^m(\w_t - \x^*)^\top\g_t^i - 2\eta \sum_{i = 1}^{m-1}\sum_{j = i+1}^m \bign{\w_t-\x^*}^\top(\g_t^i + \g_t^j) \\
& \hspace{1.5in}  + 2\frac{m(m-1)}{2}\norm{\w_{t}-\x^*}^2 + 2\eta \sum_{i = 1}^{m-1}\sum_{j = i+1}^m {\g_t^i}^\top\g_t^j.
\end{align*}
Let us denote by $n_t = \frac{\nabla f(\w_t)}{\norm{\nabla f(\w_t)}}$ the normalized gradient at point $\w_t$. 
Also let $\cH_t$ the history $\{\w_\tau,U_\tau,\o_\tau\}_{\tau = 1}^{t-1} \cup \w_{t}$ till time $t$ and  $U_t:=\{\u_t^1,\ldots,\u_t^m\}$. Then \emph{one important observation is that} the estimated gradients are nearly independent (their inner products are small):
More precisely, for any $i \neq j$, since $\u_t^i$ and $\u_t^j$ are independent, from \cref{thm:normgrad} we get: 
\begin{align*}
    \E_{U_t}&[{\g_t^i}^\top\g_t^j \mid \cH_t] = \E_{\u_i}[{\g_t^i}^\top \E_{\u_j}[\g_t^j \mid \u_t^i ]\mid \cH_t]
    \leq \frac{1}{\sqrt d}\biggn{\frac{n_t^\top n_t}{\sqrt d}} + 4\lambda_t = \frac{1}{d} + 4\lambda_t,
\end{align*}
where recall from \cref{thm:normgrad} and \cref{lem:biasgrad}, $\lambda_t \leq \frac{\beta \gamma \sqrt{d}}{\norm{\nabla f(\x)}}\bigg( 1 + 2\sqrt{\log \frac{\norm{\nabla f(\x)}}{\sqrt d \beta \gamma}} \bigg)$. Combining this with the main equation, and further applying \cref{thm:normgrad}, with a bit of algebra one can get:
\begin{align*}
 {m^2}\E_{U_t}&[\|\w_{t+1} - \x^*\|^2 \mid \cH_t] = m(\norm{\w_{t}-\x^*}^2 + \eta^2) 
 - 2\eta \sum_{i = 1}^{m-1}\sum_{j = i+1}^m \bign{\w_t-\x^*}^\top\E_{U_t}[(\g_t^i + \g_t^j) \mid \cH_t] 
\\
& + 2\frac{m(m-1)}{2}\norm{\w_{t}-\x^*}^2 - 2\eta\sum_{i = 1}^m(\w_t - \x^*)^\top\g_t^i
 + 2\eta^2 \sum_{i = 1}^{m-1}\sum_{j = i+1}^m \E_{U_t}[{\g_t^i}^\top\g_t^j \mid \cH_t]
\\
& = m^2\norm{\w_{t}-\x^*}^2 + \eta^2(m + \frac{m(m-1)}{d}) - 2\eta m^2[(\w_t-\x^*)^\top\n_t 
\\
& ~~~~~~~~~~~ + 4m^2\sqrt{d}\eta\norm{\w_t - \x^*} \lambda_t + 4m^2\sqrt{d}\eta\norm{\w_t - \x^*}\lambda_t.
\end{align*}
Further from Claim-2 of \cref{lem:cvxgrad}, and from the fact that $m<d$, we can derive:
\begin{align*}
&m^2\E_{\u_t}[\|\w_{t+1} - \x^*\|^2 \mid \cH_t] 
 \le m^2\|\w_{t} - \x^*\|^2 +  m^2(- 2\eta\frac{c\sqrt{2\epsilon}}{\sqrt {d\beta}} + 8\eta\lambda_t \sqrt{d}\|\w_{t} - \x^*\|)  + 2m\eta^2,
\end{align*}
and choosing $\gamma \le \frac{\|\nabf(\w_{t})\|}{960 \beta d\sqrt{d} \|\w_{t} - \x^*\| \sqrt{\log 480}}\sqrt{\frac{2 \epsilon}{\beta}}$, we get:
\begin{align*}
\E_{\cH_t}[\E_{\u_t}[\|\w_{t+1} - \x^*\|^2] \mid \cH_t] &\le \|\w_{t} - \x^*\|^2  - \frac{\eta\sqrt{2\epsilon}}{10\sqrt{d \beta}} + \frac{\eta\sqrt{2\epsilon}}{20\sqrt{d \beta}} + \frac{2\eta^2}{m}.
\end{align*}		
One possible choice of $\gamma$ is $\gamma = \frac{ \epsilon^{3/2}}{960 \beta d\sqrt{d} D^2\ \sqrt{\log 480}}\sqrt{\frac{2}{\beta}}$ (since $\norm{\nabla f(\x)} \geq \frac{\epsilon}{D}$ for any $\x$ s.t. $f(\x)-f(\x^*) > \epsilon$ by \cref{lem:gradf_lb}). Then following from above, we further get: 
\begin{align*}
& \E_{\cH_t}[\E_{\u_t}[\|\w_{t+1} - \x^*\|^2 \mid \cH_t]] \le \|\w_{t} - \x^*\|^2 - \frac{ (\sqrt 2 -1)m\epsilon }{400 d\beta} ~~\Big(\text{setting } \eta = \frac{m\sqrt{\epsilon}}{20\sqrt{ d \beta}}\Big)\\
		\overset{}{\implies} & \E_{\cH_T}[\|\w_{T+1} - \x^*\|^2] \le \|\w_{1} - \x^*\|^2 - \frac{  (\sqrt 2 -1)m\epsilon T}{400 d\beta}, \big(\text{summing } t = 1, \ldots T \big).
\end{align*}

Above implies, if indeed $f(\w_{\tau}) - f(\x^*) > \epsilon$ continues to hold for all $\tau = 1,2, \ldots T$, then $\E[\|\w_{T+1} - \x^*\|^2] \le 0$, for $T \ge  \frac{400 md\beta}{(\sqrt 2 -1)\epsilon}(\|\w_{1} - \x^*\|^2)$, which basically implies $\w_{T+1} = \x^*$ (i.e. $f(\w_{T+1}) = f(\x^*)$). Otherwise there must have been a time $t \in [T]$ such that $f(\w_{t}) - f(\x^*) < \epsilon$. 
The complete proof is given in \cref{app:batch}.
\end{proof}
\subsection{Improved Convergence Rates with Strong Convexity} 
\label{sec:alg_batch_strng}
We now show how to obtain a better convergence rate with an additional assumption of $\alpha$-strong convexity on $f: \cD \mapsto \R$ by simply reusing any optimal optimization algorithm for $\beta$-smooth convex functions (and hence we can use our B-NGD \cref{alg:batch}, proposed earlier).
Our proposed method Improved Batched-NGD (Alg. \cref{alg:batch_strng}) adapts a phase-wise iterative optimization approach, where inside each phase we use B-NGD as a blackbox to locate a $\epsilon_k$-optimal point in that phase, say $\w_{k+1}$, with exponentially decaying $\epsilon_k = O(\frac{1}{2^{k-1}})$. We then warm start the B-NGD algorithm in the next phase from $\w_{k+1}$ and repeat -- the idea is adapted from the similar warm starting idea proposed by \cite{SKM21}. The method yields improved $O(\log \frac{1}{\epsilon})$ convergence due to the nice property of strong convexity where nearness in function values implies nearness from the optimal $\x^*$ in $\ell_2$-norm, unlike the case for only $\beta$-smooth functions (see \cref{lem:prop_alpha}). \cref{alg:batch_strng} gives the complete detail.

\vspace{-5pt}
\begin{center}
\begin{algorithm}[h]
\caption{\textbf{Improved Batched-NGD} with Strong Convexity (ImpB-NGD)}
\label{alg:batch_strng}
\begin{algorithmic}[1] 
\STATE {\bfseries Input:} Error tolerance $\epsilon > 0$, Batch size $m$%
\STATE {\bf Initialize} Initial point: $\w_1 \in \R^d$ such that $D := \|\w_1 - \x^*\|^2$ (assume known). 
\\ Phase counts $k_\epsilon:= \lceil \log_2\big(\frac{\alpha}{\epsilon}\big)  \rceil$, $t \leftarrow \frac{800 d \beta}{(\sqrt 2 -1)\alpha}$   
\\ $\eta_1 \leftarrow \frac{m\sqrt{\epsilon_1}}{20\sqrt{ d \beta}}, \epsilon_1 = \frac{400 d\beta D}{(\sqrt 2 -1)t_1} = 1$, $t_1 = t\|\w_1-\x^*\|^2$
\\ $ \gamma_1 \leftarrow \frac{ \epsilon_1^{3/2}}{960 \beta d\sqrt d D^2\ \sqrt{\log 480}}\sqrt{\frac{2}{\beta}}$, $\m_1 = \w_1$
\STATE Update $\w_{2} \leftarrow$ \textbf{B-NGD}$\big(\w_1,\eta_1, \gamma_1,t_1\big)$
\FOR{$k = 2,3,\ldots, k_\epsilon$}
\STATE $\eta_k \leftarrow \frac{m\sqrt{\epsilon_k}}{20\sqrt{ d \beta}}, \epsilon_k = \frac{400 d\beta}{(\sqrt 2 -1)t_k}$, $t_k = 2t$
\\$ \gamma_k \leftarrow \frac{ \epsilon_k^{3/2}}{960 \beta d\sqrt d D^2\ \sqrt{\log 480}}\sqrt{\frac{2}{\beta}}$.
\STATE Update $\w_{k+1} \leftarrow$ \textbf{B-NGD}$\big(\w_{k},\eta_k, \gamma_k,t_k,m\big)$
\ENDFOR   
\STATE Return $\m_{\epsilon} = \w_{k_\epsilon+1}$
\end{algorithmic}
\end{algorithm}
\end{center}
\vspace{-5pt}

\begin{restatable}[Convergence Analysis of \cref{alg:batch_strng} for $\alpha$-strongly convex and $\beta$-Smooth Functions]{thm}{thmbatchstrng}
	\label{thm:batch_strng}
	Consider $f$ to be $\alpha$-strongly convex and $\beta$-smooth.   
	Then \cref{alg:batch_strng} returns $\E[f(\m_\epsilon)] - f(\x^*) \le \epsilon$ with sample complexity $ O\Big( \frac{d \beta}{m\alpha}(\log_2\big(\frac{\alpha}{\epsilon}\big) + \|\x_1 - \x^*\|^2) \Big)$, for any $m\leq d$. 
\end{restatable}

Due to space constraints, the proof is moved to \cref{app:batch_strng}.

\begin{rem}\textbf{\emph{[Noisy Feedback $\nu \in (0,1/2))$].}}
\label{rem:batch_noise}
Note \cref{alg:batch} and \cref{alg:batch_strng} (and consequently \cref{thm:batch} and \cref{thm:batch_strng}) work only for the noiseless feedback setting, when $\nu = 0$. However, it is easy to extend the above two algorithms for the noisy $\sign$-feedback setting (for any $\nu \in (0,0.5)$) by the resampling trick proposed in \cite{SKM21}: Precisely, the idea is to query any pair of point $(x_t,y_t)$ for $O(\frac{1}{\nu^2})$ times to recover the true $\sign$ feedback $\sign(f(\x_t) - f(\y_t))$ with high confidence, and rest of the algorithm remains as is. Clearly, this would lead to the convergence bounds of $O\Big(\frac{d\beta D}{m (0.5-\nu)^2 \epsilon} \Big)$ and $ O\Big( \frac{d \beta}{m\alpha (0.5-\nu)^2}(\log_2\big(\frac{\alpha}{\epsilon}\big) \Big)$ respectively for settings of \cref{thm:batch} and \cref{thm:batch_strng}, where the additional $O(1/(0.5-\nu)^2)$-multiplicative factor is accounted for resampling of every pairwise query in \cref{alg:batch} and \cref{alg:batch_strng}.
\end{rem}

\subsection{Lower Bound for \probll} 
\label{sec:lb_batch}

In this section, we show the convergence lower bounds for the  \probll\, problem. \cref{thm:lb_batch} shows indeed our $1/m$ rate of improvement with $m$-batch size is optimal.

\begin{restatable}[Convergence Lower Bound: $m$-\probll \, Problem]{thm}{lbbatch}
\label{thm:lb_batch}
Assume the noiseless setting $\nu = 0$ and $f: \cD \mapsto R$ be any smooth and strongly convex function. Then the $\epsilon$-convergence bound for any algorithm for the $m$-\probll\, problem is at least $\Omega(\frac{d}{m} \log \frac{1}{\epsilon})$.
\end{restatable}
\vspace{-6pt}
\begin{proof}[Proof of \cref{thm:lb_batch}]
We first require to show $\Omega(d \log \frac{1}{\epsilon})$ convergence lower bound for the \probsgn\, problem for smooth and strongly convex functions, as we prove below \footnote{It is important to note that Theorem 1 of \citep{Jamieson12} claims to yield an $\Omega(d \log \frac{1}{\epsilon})$ lower bound for the same problem (for their setting $\kappa=1$). However, they still need to assume their noise parameter $\mu$, which is equivalent to $\mu = 1/2 - \nu$ in our case, satisfies $\mu \leq 1/2$, which is equivalent to assuming $\gamma > 0$ in our case. So their lower bound is information-theoretic owning to the noisy sign feedback, but not an optimization-based lower bound. Precisely, their lower bound does not apply in the noiseless setting $\mu = 1/2$ (or $\nu = 0$).}:

\begin{restatable}[Convergence Lower Bound for \probsgn \, Problem]{lem}{lbsingle}
\label{thm:lb_single}
Let $f: \cD \mapsto R$ be any smooth and strongly convex function. Then the $\epsilon$-convergence bound for any algorithm for the \probsgn\, problem is at least $\Omega(d \log \frac{1}{\epsilon})$.
\end{restatable}
To proof above, assume $\cD = \cB_d(1)$ is the unit ball in dimension $d$ and let $N(\cD, \epsilon, \norm{}_2)$ be the $\epsilon$ packing of $\cD$ w.r.t. $\ell_2$-norm \citep{wu2016packing}. Let $f(x) = \norm{\x - \x^*}_2^2$, and the adversary can select $\x^*$ arbitrarily as any $\x^* \in N(\cD, \epsilon, \norm{}_2)$. Then note any single pair of sign feedback can allow the learner to remove at most half of the point in decision space $\cD$, so after $t$ number of pairwise sign feedback, the adversary still has the choice to select $\x^*$ from $\frac{1}{2^t}\abs{N(\cD, \epsilon, \norm{}_2)} \leq \frac{1}{2^t}\bign{\frac{3}{\epsilon}}^d$ many numbers of points. This yields the desired $\epsilon$-convergence sample complexity lower bound of $\Omega(d \log \frac{1}{\epsilon})$, as the learner would need to make at least $t \geq d \log \frac{1}{\epsilon}$ many pairwise $\sign$ queries before the adversary would be left with atmost $\frac{1}{2^t}\Bign{\frac{3}{\epsilon}}^d \leq 1$ choice for $\x^*$.  

Having equipped with the lower bound of \cref{thm:lb_single}, this immediately implies the desired lower bound of \cref{thm:lb_batch} as one can hope to get an improved convergence bound of at most $1/m$-multiplicative factor, even when $m$ rounds are merged into a single round.  
\end{proof}

\begin{rem}
It is worth noting that the lower bound above assumes $f$ to be both strongly convex and smooth which yields a convergence lower bound of  $\Omega(d\log \frac{1}{\epsilon})$ in the first place. However, it remains an open problem if we can obtain $\Omega(\frac{d}{\epsilon})$ lower bound for the class of just smooth functions to match the upper bound of \cref{thm:batch} for only smooth convex functions (without strong convexity). 
\end{rem}


\section{Battling (Multiwise-Winner) Convex Optimization with Sign Feedback}
\label{sec:argmin}

In this section, we investigate the \probarg\, problem. Recall from \cref{sec:problem}, in this case at each iteration $t$, the learner can query a set $S_t$ of $m$ points $S_t = (\x_t^1,\x_t^2,\ldots,\x_t^m)$ and gets to see the $\arg\min$ feedback of the $m$-points: $\{o_t \in [m]\}$ such that: 
\[
Pr[o_t = \arg\min(f(\x_t^1), f(\x_t^2), \ldots, f(\x_t^m))] = 1 - \nu. 
\]
As before, We considered the noiseless case $(\nu = 0)$ first, and present two algorithms for smooth and strongly convex optimization settings. The interesting fact is, in this case, we could only show $O(\frac{1}{\log m})$ improved convergence rate in this feedback model compared to the $O(\frac{1}{m})$ in the \probll\, setting (see \cref{thm:argmin} and \cref{thm:argmin_strng}). In fact, a more interesting fact is we also show that $O(\frac{1}{\log m})$ improvement is the best we can hope for in this feedback model, proving a matching convergence lower bound (see \cref{thm:lb_argmin}).

\subsection{Proposed Algorithm: Battling-NGD} 
\label{sec:alg_argmin}

\textbf{Useful Notations. } We denote by $V_n = \{(\pm 1)^n\}$, for any $n \in \N_+$. 
Clearly $\abs{V_n} = 2^n$. 
Let $\cG(V_)$ be the graph with vertex set $V_n \subseteq \{\pm 1\}^n$ and there exists an (undirected) edge between two nodes $\v$ and $\tilde \v$ iff $\v$ and $\tilde \v$ only differs sign in one of the $n$ coordinates, i.e. $\exists k \in [n], ~v(k) = \tilde v(k)$ and $v(k') = \tilde v(k')$ for any $k' \neq k$. 
Clearly the number of neighboring nodes of any vertex $\v \in V_n$ in graph $\cG$ is $\abs{\cN(\v,\cG)} = n$. In other words, the degree of any node in graph $\cG$ is $n$. We show an example for $n=3$ in the right figure. 
Also, let us define $\ell_m = \lfloor\log m\rfloor$. 

\textbf{Algorithm Description. } 
There are three novelties in our algorithmic idea: 
$(i).$ Structure of the query sets, $(ii). $ \emph{One-vs-All} feedback idea and $(iii).$ Extracting $\log m$ batched sign feedback. We explain them in more detail below.

\begin{minipage}[t]{0.67\linewidth}
\textbf{$(i).$ Structured Query Sets: }
As before, the algorithm maintains a current point $\w_t$ (initialized to $\w_1 \in \cD$). At each time $t$, it queries a set $S_t$ of $m$ points around $\w_t$ such that for every point $\x \in S_t$, there exists exactly $\ell_m$ neighboring points which are \emph{symmetrically opposite to $\x$ in exactly one of the realization of $\u_t^i$s}: More precisely, at each time $t$, the algorithm first samples $\ell_m$ vectors $\u_t^i \overset{\text{iid}}{\sim} \text{Unif}(\cS_d(\frac{1}{\sqrt{\ell_m}}))$ independently, $i \in [\ell_m]$. Let 
\end{minipage}
\hspace{-5pt}
\begin{minipage}[t]{0.33\linewidth}
  \vspace{-4ex} 
   \includegraphics[width=0.99\textwidth]{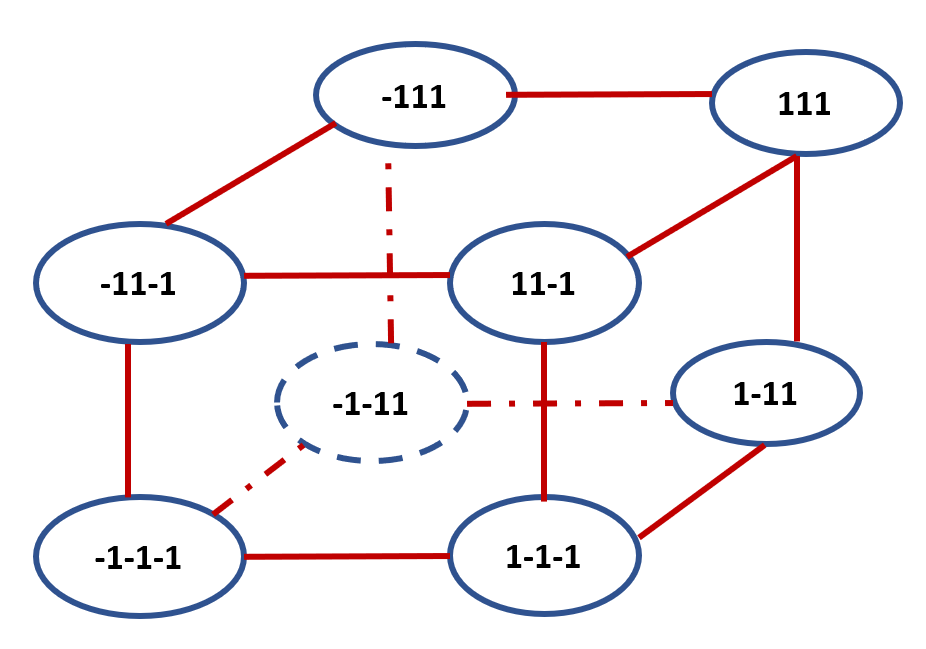}
\end{minipage}

$U_t = [\u_t^1,\ldots,\u_t^{\ell_m}] \in \R^{d \times \ell_m}$, 
  and define $S_t = \{\w_t + \gamma U_t\v \mid \v \in V_{\ell_m}\}$. Note that by construction indeed $S_t = 2^{\ell_m} \leq m$. Further, note for any point $\x = \w_t + \gamma U_t\v \in S_t$ there exists exactly $\ell_m$ symmetrically opposing points $\x'_i = \w_t + \gamma U_t\v'_i \in S_t$, for all $\v'_i \in \cN(\v,\cG)$ such that $\frac{(\x - \x'_i)}{2\gamma v_i} = \u_t^i, ~i \in [\ell_m]$. 
Given any such point $\x_\v:= \w_t + \gamma U_t\v$, let us denote by the set $\cN(\x_\v) = \{\w_t + \gamma U_t\v'_i \mid \v'_i \in \cN(\v,\cG)\}$ of all symmetrically opposing points of $\x$ in $S_t$ around $\w_t$ which differs in exactly one of the realization of $\u_t^i$s. This property will play a very crucial role in our analysis, as will see in the convergence proof of \cref{alg:argmin} (see proof of \cref{thm:argmin}).    
Upon constructing the set $S_t$, the algorithm queries the $m$-subset $S_t$ and receives the winner feedback $o_t \in [m]$. 

\vspace{6pt}
\textbf{$(ii).$ One-vs-All Feedback Idea: } Note by definition of the winner feedback model, $f(\x_t^{o_t}) < f(\x_t^i), ~\forall i \in [m], i \neq o_t$. Thus clearly, $\sign\bign{ f(\x_t^{o_t}) < f(\x_t^i)} = -1$. So one may essentially recover exactly $m-1$ pairwise $\sign$-feedback. 

\vspace{6pt}
\textbf{$(iii).$ Extracting $\log m$ Batched Sign Feedback: }
However, there are inherent dependencies among these pair of points and most of these extracted sign feedback are redundant. We precisely identify $O(\log m)$ specific winner-vs-loser pairs and use their pairwise sign feedback to obtain a normalized gradient estimate. Let us denote by $\x_t^{o_t} = \w_t + \gamma U_t\v$ for some $\v \in V_{\ell_m}$. Then we choose all the symmetrically opposing pairs $(\x_t^{o_t},\y_t^i)$ for all $\y_t^i \in \cN(\x_t^{o_t})$ (as described in $\#(i)$ above), and extract the corresponding $\ell_m$ pairwise sign feedback. The setting then can simply reduce back to the $O(\log m)$-batched sign feedback setting and one use the similar algorithmic idea of \cref{alg:batch}.

More precisely, the algorithm finds the $\ell_m$ normalized gradient estimates $g_t^i = o_t^i\v_t^i$ for all $i \in [\ell_m]$, where $o_t^i = \sign\bign{f(\x_t^{o_t}) - f(\y_t^i)} = -1$ is the sign feedback of the $i$-th (winner-loser) pair and $\v_t^i = v_i\u_t^i \sim \text{Unif}(\cS_d(\frac{1}{\sqrt{\ell_m}}))$ is the corresponding `scaled-unit' direction. Finally, we update the running prediction using the normalized gradient descent technique $\w_{t+1} \leftarrow \w_t - \eta \g_t$, using the aggregated descent direction $\g_t = \frac{1}{\ell_m}\sum_{i = 1}^{\ell_m}\g_t^i$. As before, we also maintain a running minimum $\m_t$ which keeps track of $\min\{\w_1,\ldots,\w_t\}$. The complete algorithm is given in \cref{alg:argmin}. 
\vspace{-10pt}
\begin{center}
	\begin{algorithm}[h]
		\caption{\textbf{Battling-NGD}} 
		\label{alg:argmin}
		\begin{algorithmic}[1] 
			\STATE {\bfseries Input:} Initial point: $\w_1 \in \cD$, Learning rate $\eta$, Perturbation parameter $ \gamma$, Query budget $T$ (depends on error tolerance $\epsilon$), Batch-size $m$. Define $\ell_m := \lfloor\log m\rfloor$ and $\tilde m := 2^{\ell_m} \leq m$. 
            \STATE {\bf Initialize} Current minimum $\m_1 = \w_1$
			\FOR{$t = 1,2,3,\ldots, T$}
			\STATE Sample $\u_t^1,\u_t^2,\ldots\u_t^{\ell_m} \overset{\text{iid}}{\sim} \text{Unif}(\cS_d(\frac{1}{\sqrt \ell_m}))$. Denote $U_t:= [\u_t^1,\ldots,\u_t^{\ell_m}] \in \R^{d \times \ell_m}$  
                \STATE Define $S_t:= \{\w_t + \gamma U_t\v \mid \v \in V_{\ell_m}\}$ (see definition of $V_{\ell_m}$ in the description)
                \STATE Play the $m$-subset $S_t$ 
    \STATE Receive the winner feedback $o_t = \arg\min(f(\x_t^1), f(\x_t^2), \ldots, f(\x_t^{\tilde m}))$
    \STATE Update $\tw_{t+1} \leftarrow \w_t - \eta \g_t$, 
			where $\g_t = \frac{1}{\ell_m}\sum_{i=1}^{\ell_m}\g_t^i$, $\g_t^i = -v_i \u_t^i$
    \STATE Project $\w_{t+1} = \arg\min_{\w \in \cD}\norm{\w - \tw_{t+1}}$
    \STATE Query the pair $(\m_{t}, \w_{t+1})$ and receive $\sign\bign{ f(\m_{t}) - f(\w_{t+1}) }$. 
    \STATE $\m_{t+1} \leftarrow 
			\begin{cases}
				\m_{t} ~\text{ if } \sign\bign{ f(\m_{t}) - f(\w_{t+1}) }<0\\
				\w_{t+1} ~\text{ otherwise }  
			\end{cases}$
    \ENDFOR   
   \STATE Return $\m_{T+1}$
		\end{algorithmic}
	\end{algorithm}
\end{center}
\vspace{-10pt}
\begin{restatable}[Convergence Analysis of \cref{alg:argmin} for $\beta$-Smooth Functions]{thm}{thmargmin}
\label{thm:argmin}
Consider $f$ to be $\beta$ smooth. Suppose Alg. \ref{alg:argmin} is run with $\eta = \frac{\ell_m\sqrt{\epsilon}}{20\sqrt{ d \beta}}$,  $ \gamma = \frac{ \epsilon^{3/2}}{960 \beta d\ell_m\sqrt {d \ell_m} D^2\ \sqrt{\log 480}}\sqrt{\frac{2}{\beta}}$ and $T_\epsilon =  O\Big(\frac{d\beta D}{\epsilon \ell_m} \Big)$, where $\|\w_{1} - \x^*\|^2 \leq D$, $\ell_m = \lfloor\log m\rfloor \leq d$. Then \cref{alg:argmin} returns $\E[f(\m_{T+1})] - f(\x^*) \le \epsilon$ with sample complexity $2T_\epsilon$. 
\end{restatable}

\begin{proof}[Proof Sketch of \cref{thm:argmin}]
Due to space limitations, the complete proof is deferred to \cref{app:thm_argmin}. The key idea relies on the idea of constructing the structured query set $S_t$ which allows us to derive $O(\log m)$ winner-vs-loser sign feedback from the $m$-multiwise winner $o_t$, each of which results in an \emph{independent} (normalized) gradient estimate $\g_t^i$, $\forall i \in [\ell_m]$. Note here the independence of $\g_t^i$s are crucial, which was possible due to the special structure of the query set $S_t$. We prove the formal statement in \cref{thm:new_normgrad}, \cref{app:lem_argmin}. The final stretch of the proof relies on exploiting the aggregated normalized gradient $\g_t = \frac{1}{\ell_m}\sum_{i = 1}^{\ell_m}\g_t^i$, similar to batched feedback model that yields the final $O(\frac{1}{\ell_m})$ factor improvement in the sample complexity of the convergence analysis.  
\end{proof}

\subsection{Improved Convergence Rates with Strong Convexity} 
\label{sec:alg_argmin_strng}

It is easy to argue that, with the additional assumption of strong convexity, we can again obtain an improved convergence rate of $O(\frac{d}{\log m}\log \frac{1}{\epsilon})$ similar to batched setting (\cref{sec:alg_batch_strng}). Following the same `phase-wise progress with warm starting' idea of \cref{alg:batch_strng} with now using Battling-NGD (\cref{alg:argmin}) as the underlying black-box for an algorithm of \probarg \, the problem for smooth convex functions, we can design an algorithm, say \textbf{Improved Battling-NGD}, to achieve improved $O(\log \frac{1}{\epsilon})$ convergence rate with strong convexity. We omit the specific details for brevity which can be easily inferred by combining \cref{alg:batch_strng} and \cref{alg:argmin}. \cref{thm:argmin_strng} gives the convergence rate of the above approach. Further details are deferred to \cref{app:thm_argmin_strng}.

\begin{restatable}[Improved Convergence Rate for $\alpha$-strongly convex and $\beta$-Smooth Functions]{thm}{thmargstrng}
	\label{thm:argmin_strng}
	Consider $f$ to be $\alpha$-strongly convex and $\beta$-smooth and let $\ell_m = \lfloor\log m\rfloor \leq d$.   
	Then \emph{Improved Battling-NGD} returns an $\epsilon$-optimal point within $ O\Big( \frac{d \beta}{\alpha \ell_m}(\log_2\big(\frac{\alpha}{\epsilon}\big) + \|\x_1 - \x^*\|^2) \Big)$ many multiwise queries.
\end{restatable}

\vspace{-5pt}
\textbf{Noisy Argmin Feedback $\nu \in (0,1/2)$. } This setting can be handled using the same `\emph{resampling idea}' explained in \cref{rem:batch_noise}. This would respectively yield convergence bounds of   $O\Big(\frac{d\beta D}{(0.5-\nu)^2 \epsilon \log m } \Big)$ and $ O\Big( \frac{d \beta}{\alpha (0.5-\nu) ^2 \log m}(\log_2\big(\frac{\alpha}{\epsilon}\big) + \|\x_1 - \x^*\|^2) \Big)$ for the settings of \cref{thm:argmin} and \cref{thm:argmin_strng}.

\subsection{Lower Bound: \probarg} 
\label{sec:lb_argmin}
\vspace{-2pt}
\begin{restatable}[Convergence Lower Bound for \probarg \, Problem]{thm}{lbargmin}
\label{thm:lb_argmin}
Let $f: \cD \mapsto R$ be any smooth and strongly convex function. Then the $\epsilon$-convergence bound for any algorithm for the $m$-\probarg\, problem is at least $\Omega(\frac{d}{\log m} \log \frac{1}{\epsilon})$.
\end{restatable}
\vspace{-5pt}
\begin{proof}
We appeal back to the lower bound derivation idea of \cref{thm:lb_batch} to derive the lower bound of \cref{thm:lb_argmin}.
Same as before, let us assume $\cD = \cB_d(1)$ and $N(\cD, \epsilon, \norm{}_2)$ be the $\epsilon$ packing of $\cD$ w.r.t. $\ell_2$-norm. Also assume $f(x) = \norm{\x - \x^*}_2^2$, and the adversary can select $\x^*$ arbitrarily as any $\x^* \in N(\cD, \epsilon, \norm{}_2)$. 

Then in this battling (multiwise-winner) feedback model, note that a single $m$-point subsetwise query can allow the learner to remove at most $1/m$ portions the points in decision space $\cD$. Then after $t$ number of such queries, the adversary will still have the choice to select $\x^*$ from $\frac{1}{m^t}\abs{N(\cD, \epsilon, \norm{}_2)} \leq \frac{1}{2^t}\Bign{\frac{3}{\epsilon}}^d$ many numbers of points. This immediately yields the desired $\Omega(\frac{d}{\log m} \log \frac{1}{\epsilon})$ sample complexity lower bound for the \probarg\, setting, as the learner would need to make at least $t \geq \frac{d}{\log m} \log \frac{1}{\epsilon}$ many $m$-multiwise queries before the adversary would be left with at most $\frac{1}{m^t}\Bign{\frac{3}{\epsilon}}^d \leq 1$ choice for $\x^*$.  
\end{proof}
\vspace{-6pt}
Note, \cref{thm:argmin} shows that our proposed algorithms actually yield optimal convergence rate in terms of $m$, but it is still an open problem to see if one can prove a matching $\Omega(\frac{d}{\epsilon \log m})$ convergence lower bound for the class of smooth functions (without strong convexity).






\vspace{-5pt}
\section{Experiments}
\label{sec:expts}
\vspace{-5pt}
In this section, we provide empirical evaluation of our proposed methods to compare the convergence rates with different types of feedback models (1) Single $\sign$\ Feedback, (2) $m$-Batched $\sign$\ Feedback, and (3) $m$-Argmin Feedback. We run experiments in the following settings:

\textbf{Algorithms. } We compare three algorithms, (1) NGD, (2) Batched-NGD ($m$-NGD) and (3) Battling-NGD ($s$-NGD) for the above three different types of feedback.

%

\textbf{Experiment} Tradeoff between the Query completely ($T$) vs SubOptimality Gap $(f(\w_t)-f(\x^*))$ for different types of multiway preference feedback: 

\begin{figure}[h]
\begin{center}
  \includegraphics[width=0.45\linewidth,trim=0 0 135pt 0, clip]{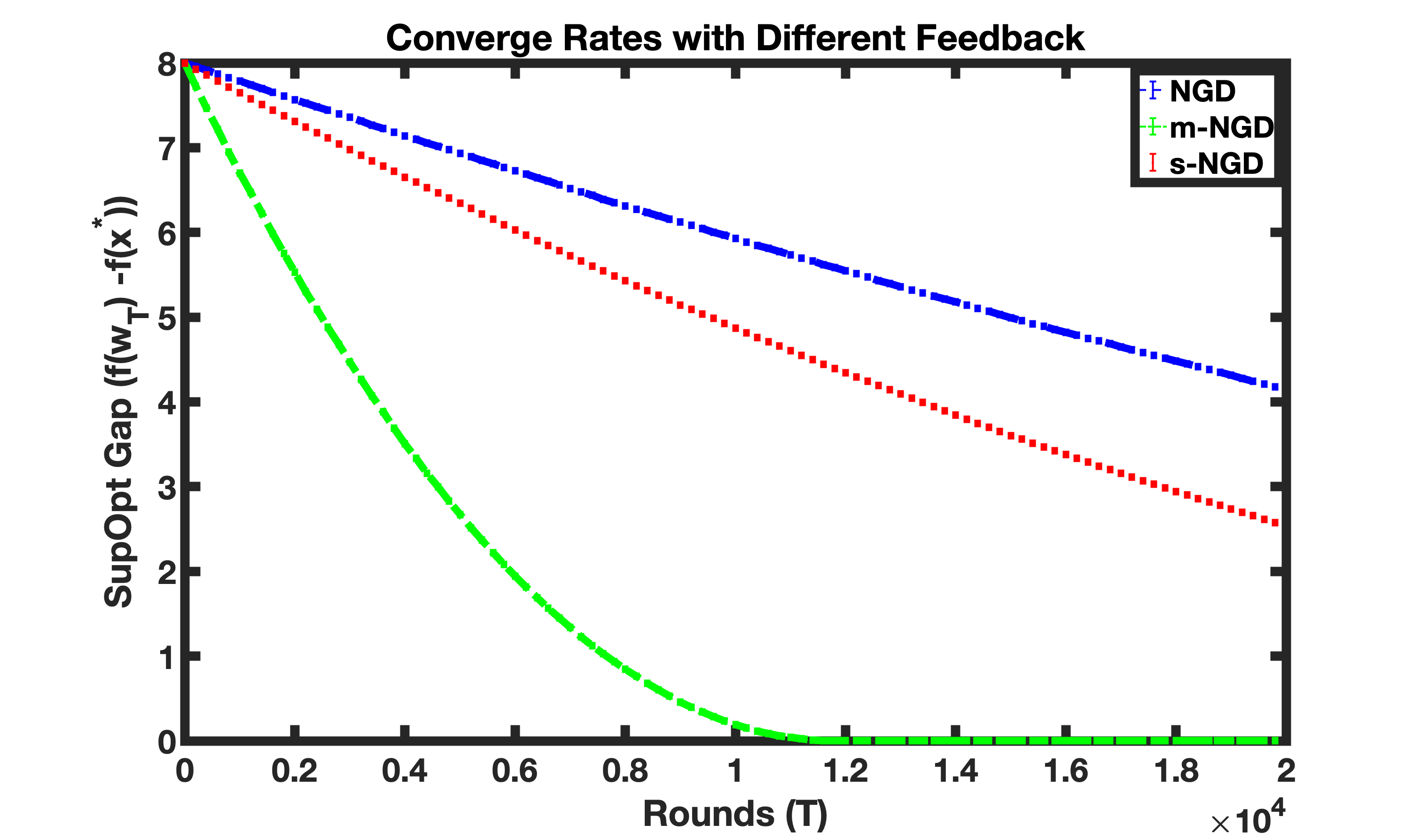}
  \includegraphics[width=0.45\linewidth,trim=0 0 135pt 0, clip]{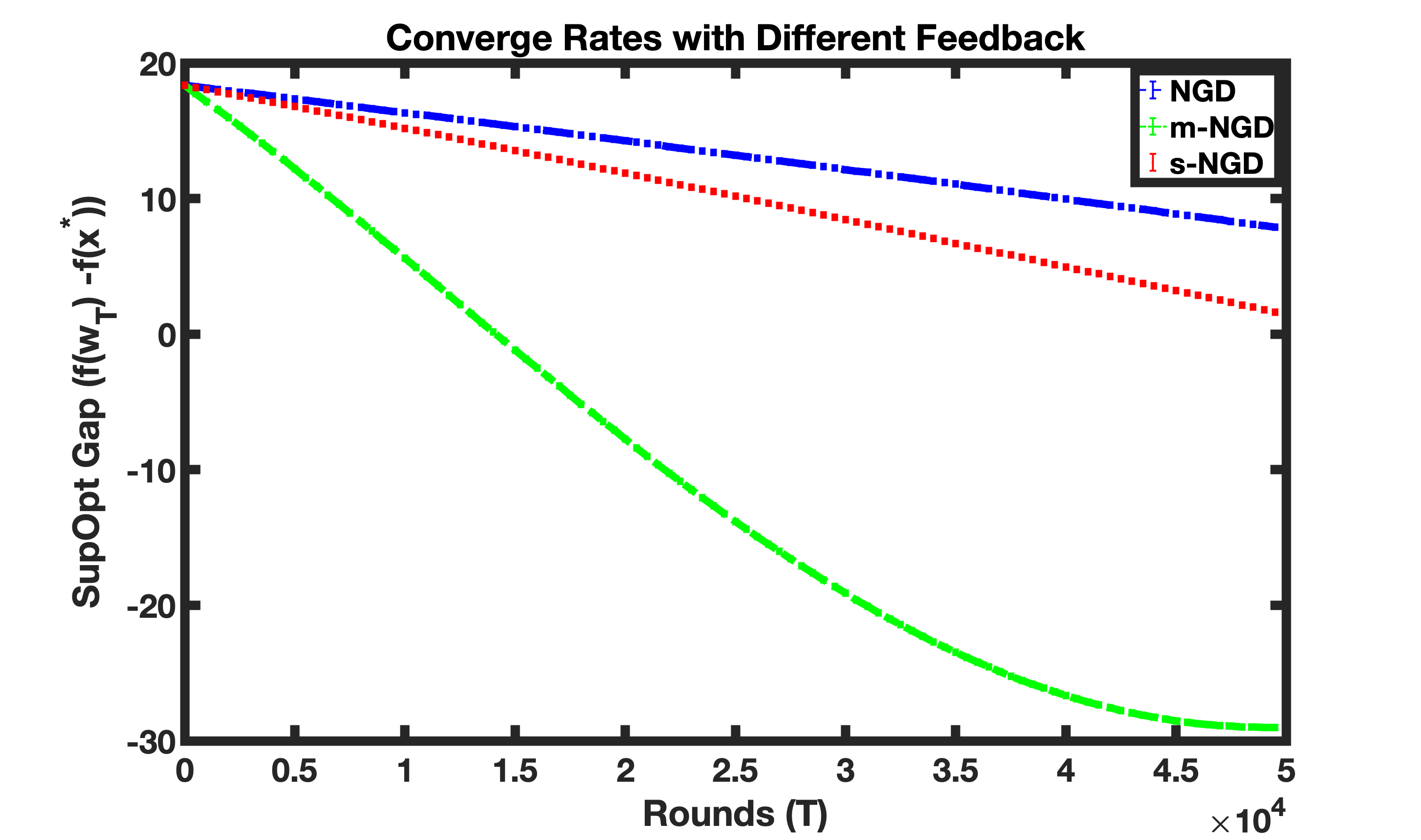}
  \vspace{-5pt}
  \caption{Query completely ($T$) vs SubOptimality Gap $(f(\w_t)-f(\x^*))$ for (1) Single $\sign$\ Feedback (NGD), (2) $m$-Batched $\sign$\ Feedback ($m$-NGD), and (3) $m$-Argmin $\sign$\ Feedback ($s$-NGD). First plot uses $f(x) = \|\x\|_2^2$. The second plot uses $f(x) = 3 + \sum_{i = 1}^d sin(x_i)$. We set $d = 32$, $m = 6$ and initialized $x_0(i) = 0.5, ~\forall i \in [d]$.}
  \label{fig:expt1}
\end{center}
\end{figure}
\textbf{Observations:} \cref{fig:expt1} corroborates our results showing that indeed we get the fastest convergence for $m$-Batched $\sign$ Feedback, followed by $m$-Multiway Feedback and single $\sign$ feedback leads to the slowest convergence in each case. Further, note $f(x^*)$ is respectively $0$ and $-29$ for the above two functions, and the algorithms indeed tend to converge to the true minimum over time. Since $m$-NGD converges the fastest, note it already converged to the true minimum in both settings. 


We report some additional experiments to reflect the tradeoff between convergence rate vs size of multiway preferences, for both $m$-Batched $\sign$ and $m$-Argmin feedback. We use the function $f:\R^{32} \mapsto \R$ such that
$f(x) = \|\x\|_2^2 + \frac{1}{2}\|\x\|_1^2$. So we have $d = 32$. 

\begin{figure}[h!]
	\begin{center}
		\includegraphics[width=0.45\linewidth,trim=0 0 135pt 0, clip]{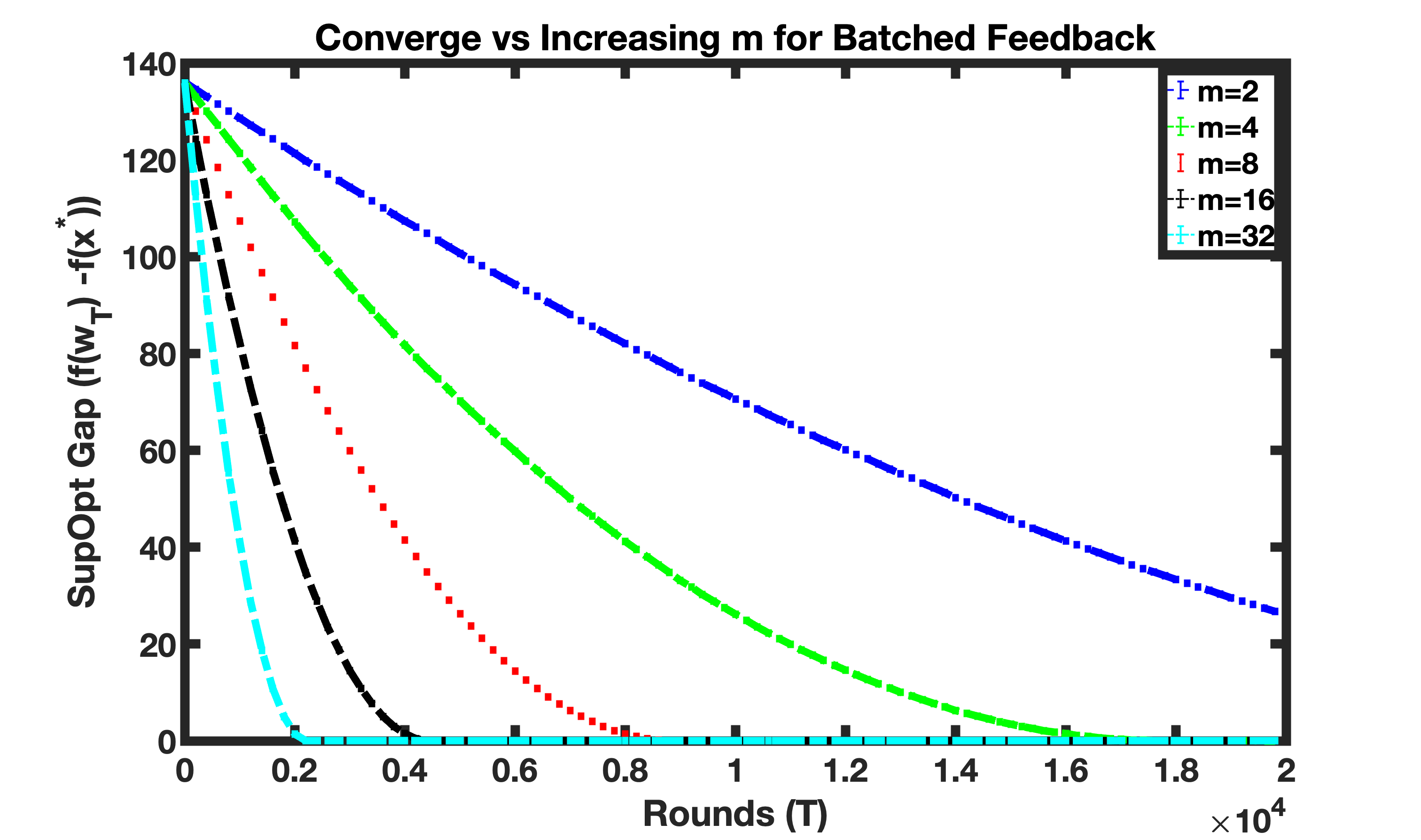}
		\vspace{10pt}
		\includegraphics[width=0.45\linewidth,trim=0 0 135pt 0, clip]{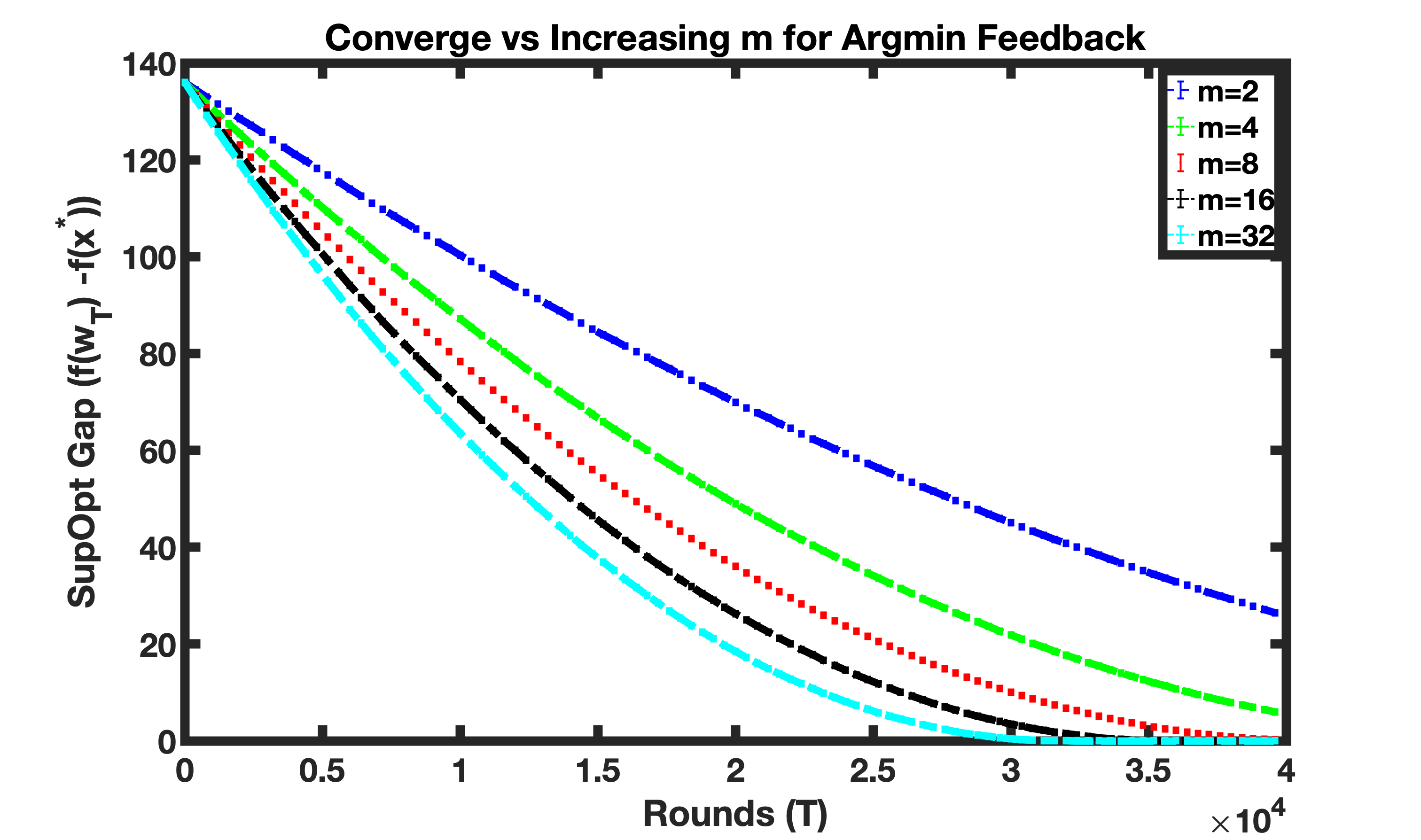}
		\vspace{-20pt}
		\caption{(left) Convergence rate vs time with increasing $m$ (right) Convergence rate vs time with increasing $m$. We use $f(x) = \|\x\|_2^2 + \frac{1}{2}\|\x\|_1^2$ for both cases, with $d = 32$, and initialized $x_0(i) = 0.5, ~\forall i \in [d]$.}
		\label{fig:expt2}
	\end{center}
\end{figure}

\textbf{Observations:} \cref{fig:expt2} again corroborates with our inferences on the the tradeoff between convergence rates vs increasing $m$ for the $m$-batched and $m$-argmin feedback. Indeed the rate of convergence decays linearly as $O(1/m)$ for $m$-Batched $\sign$ Feedback, as derived in \cref{thm:batch}, \cref{thm:batch_strng}; whereas with $m$-argmin feedback the decay is only logarithmic $O(frac{1}{\log m})$ in $m$, as derived in \cref{thm:argmin}, \cref{thm:argmin_strng}. The algorithms converge to the true minimum $f(\x^*)=0$ in both settings, although the rate of convergence is much faster in the left case (i.e with $m$-batched $\sign$ feedback).

\vspace{-7pt}
\section{Perspective}
\label{sec:concl}
\vspace{-6pt}
We address the problem of convex optimization with multiway preference feedback, where the learner can only receive relative feedback of a subset of queried points and design gradient descent-based algorithms with fast convergence rates for smooth and strongly convex functions. In particular, we worked with batched and argmin-type $m$-multiway preferences and designed algorithms with optimal convergence dependencies on $m$. Our work is the first to study and analyze the problem of convex optimization with multiway preferences. 

A natural extension of this work could be to understand if one can work with the class of any arbitrary convex functions (beyond the smoothness assumption which is crucially used to derive the normalized gradient estimates in the current algorithms). Investigating our problem setup to a more general class of preference functions could be useful to understand what is the right rate of improvement one can hope for with parallelism \cite{Ren+18}. 
Another interesting direction could be to generalize the setting to an online model, where the underlying functions can vary across time. Can we even hope to achieve sublinear convergence (or regret bounds) for such cases?

\newpage

\bibliographystyle{abbrvnat} 
\bibliography{dueling-refs,opt-refs}

\begin{thebibliography}{22}
\providecommand{\natexlab}[1]{#1}
\providecommand{\url}[1]{\texttt{#1}}
\expandafter\ifx\csname urlstyle\endcsname\relax
  \providecommand{\doi}[1]{doi: #1}\else
  \providecommand{\doi}{doi: \begingroup \urlstyle{rm}\Url}\fi

\bibitem[Ailon et~al.(2014)Ailon, Karnin, and Joachims]{Ailon+14}
N.~Ailon, Z.~S. Karnin, and T.~Joachims.
\newblock Reducing dueling bandits to cardinal bandits.
\newblock In \emph{ICML}, volume~32, pages 856--864, 2014.

\bibitem[Bengs et~al.(2021)Bengs, Busa-Fekete, El~Mesaoudi-Paul, and
  H{\"u}llermeier]{Busa21survey}
V.~Bengs, R.~Busa-Fekete, A.~El~Mesaoudi-Paul, and E.~H{\"u}llermeier.
\newblock Preference-based online learning with dueling bandits: A survey.
\newblock \emph{Journal of Machine Learning Research}, 2021.

\bibitem[Bubeck(2014)]{bubeckbook}
S.~Bubeck.
\newblock Convex optimization: Algorithms and complexity.
\newblock \emph{arXiv preprint arXiv:1405.4980}, 2014.

\bibitem[Chen et~al.(2017)Chen, Gopi, Mao, and Schneider]{ChenSoda+17}
X.~Chen, S.~Gopi, J.~Mao, and J.~Schneider.
\newblock Competitive analysis of the top-k ranking problem.
\newblock In \emph{Proceedings of the Twenty-Eighth Annual ACM-SIAM Symposium
  on Discrete Algorithms}, pages 1245--1264. SIAM, 2017.

\bibitem[Chen et~al.(2018)Chen, Li, and Mao]{ChenSoda+18}
X.~Chen, Y.~Li, and J.~Mao.
\newblock A nearly instance optimal algorithm for top-k ranking under the
  multinomial logit model.
\newblock In \emph{Proceedings of the Twenty-Ninth Annual ACM-SIAM Symposium on
  Discrete Algorithms}, pages 2504--2522. SIAM, 2018.

\bibitem[Ghoshal and Saha(2022)]{SGho21}
S.~Ghoshal and A.~Saha.
\newblock Exploiting correlation to achieve faster learning rates in low-rank
  preference bandits.
\newblock In \emph{International Conference on Artificial Intelligence and
  Statistics}, pages 456--482. PMLR, 2022.

\bibitem[Hazan(2019)]{hazanbook}
E.~Hazan.
\newblock Introduction to online convex optimization.
\newblock \emph{arXiv preprint arXiv:1909.05207}, 2019.

\bibitem[Jamieson et~al.(2012)Jamieson, Nowak, and Recht]{Jamieson12}
K.~G. Jamieson, R.~Nowak, and B.~Recht.
\newblock Query complexity of derivative-free optimization.
\newblock In \emph{Advances in Neural Information Processing Systems}, pages
  2672--2680, 2012.

\bibitem[Nesterov(2003)]{nesterovbook}
Y.~Nesterov.
\newblock \emph{Introductory lectures on convex optimization: A basic course},
  volume~87.
\newblock Springer Science \& Business Media, 2003.

\bibitem[Ren et~al.(2018)Ren, Liu, and Shroff]{Ren+18}
W.~Ren, J.~Liu, and N.~B. Shroff.
\newblock P{AC} ranking from pairwise and listwise queries: Lower bounds and
  upper bounds.
\newblock \emph{arXiv preprint arXiv:1806.02970}, 2018.

\bibitem[Saha and Gopalan(2018)]{SG18}
A.~Saha and A.~Gopalan.
\newblock Battle of bandits.
\newblock In \emph{Uncertainty in Artificial Intelligence}, 2018.

\bibitem[Saha and Gopalan(2019{\natexlab{a}})]{SGrank18}
A.~Saha and A.~Gopalan.
\newblock Active ranking with subset-wise preferences.
\newblock \emph{International Conference on Artificial Intelligence and
  Statistics (AISTATS)}, 2019{\natexlab{a}}.

\bibitem[Saha and Gopalan(2019{\natexlab{b}})]{SGwin18}
A.~Saha and A.~Gopalan.
\newblock P{AC} battling bandits in the plackett-luce model.
\newblock In \emph{Algorithmic Learning Theory}, pages 700--737,
  2019{\natexlab{b}}.

\bibitem[Saha and Gopalan(2020{\natexlab{a}})]{SGpac20}
A.~Saha and A.~Gopalan.
\newblock From pac to instance-optimal sample complexity in the plackett-luce
  model.
\newblock In \emph{International Conference on Machine Learning}, pages
  8367--8376. PMLR, 2020{\natexlab{a}}.

\bibitem[Saha and Gopalan(2020{\natexlab{b}})]{SGrum20}
A.~Saha and A.~Gopalan.
\newblock Best-item learning in random utility models with subset choices.
\newblock In \emph{International Conference on Artificial Intelligence and
  Statistics}, pages 4281--4291. PMLR, 2020{\natexlab{b}}.

\bibitem[Saha and Krishnamurthy(2022)]{SK21}
A.~Saha and A.~Krishnamurthy.
\newblock Efficient and optimal algorithms for contextual dueling bandits under
  realizability.
\newblock In \emph{International Conference on Algorithmic Learning Theory},
  pages 968--994. PMLR, 2022.

\bibitem[Saha et~al.(2021)Saha, Koren, and Mansour]{SKM21}
A.~Saha, T.~Koren, and Y.~Mansour.
\newblock Dueling convex optimization.
\newblock In \emph{International Conference on Machine Learning}, pages
  9245--9254. PMLR, 2021.

\bibitem[Sui et~al.(2017)Sui, Zhuang, Burdick, and Yue]{sui17}
Y.~Sui, V.~Zhuang, J.~Burdick, and Y.~Yue.
\newblock Multi-dueling bandits with dependent arms.
\newblock In \emph{Conference on Uncertainty in Artificial Intelligence},
  UAI'17, 2017.

\bibitem[Sui et~al.(2018)Sui, Zoghi, Hofmann, and Yue]{sui2018advancements}
Y.~Sui, M.~Zoghi, K.~Hofmann, and Y.~Yue.
\newblock Advancements in dueling bandits.
\newblock In \emph{IJCAI}, pages 5502--5510, 2018.

\bibitem[Wu and Liu(2016)]{DTS}
H.~Wu and X.~Liu.
\newblock Double {T}hompson sampling for dueling bandits.
\newblock In \emph{Advances in Neural Information Processing Systems}, pages
  649--657, 2016.

\bibitem[Wu(2016)]{wu2016packing}
Y.~Wu.
\newblock Packing, covering, and consequences on minimax risk.
\newblock \emph{Course Lecture notes for ECE598: Information-theoretic methods
  for high-dimensional statistics}, 2016.

\bibitem[Yue and Joachims(2009)]{Yue+09}
Y.~Yue and T.~Joachims.
\newblock Interactively optimizing information retrieval systems as a dueling
  bandits problem.
\newblock In \emph{Proceedings of the 26th Annual International Conference on
  Machine Learning}, pages 1201--1208, 2009.

\end{thebibliography}
\newpage

\appendix

\onecolumn{
	
\section*{\centering\Large{Supplementary: \papertitle}}
	
\allowdisplaybreaks

\section{Appendix for \cref{sec:batch}}

\subsection{Proof of \cref{thm:batch}}
\label{app:batch}

\thmbatch*

\begin{proof}[Proof of \cref{thm:batch}]
We start by noting that in this case: 

\begin{align}
\label{eq:thm1prf_init}
\norm{\w_{t+1} - \x^*}^2 \leq \norm{\tilde \w_{t+1} - \x^*}^2  =  \norm{\w_{t} -\frac{\eta}{m}\sum_{i = 1}^m\g_t^i - \x^*}^2, 
\end{align}

where the first inequality holds since projection reduces the distance to optimal $\x^*$. This further leads to

\begin{align*}
& {m^2}\norm{\w_{t+1} - \x^*}^2  
= \sum_{i=1}^m\norm{\w_{t} -\eta\g_t^i - \x^*}^2 + 2\sum_{1\leq i < j \leq K} \big(\w_{t} -\eta\g_t^i - \x^*\big)^\top\big(\w_{t} -\eta\g_t^i - \x^*\big)\\
& = m(\norm{\w_{t}-\x^*}^2 + \eta^2) - 2\eta\sum_{i = 1}^m(\w_t - \x^*)^\top\g_t^i - 2\eta \sum_{i = 1}^{m-1}\sum_{j = i+1}^m \bign{\w_t-\x^*}^\top(\g_t^i + \g_t^j) + 2\frac{m(m-1)}{2}\norm{\w_{t}-\x^*}^2 \\
& \hspace{4in} + 2\eta \sum_{i = 1}^{m-1}\sum_{j = i+1}^m {\g_t^i}^\top\g_t^j.
\end{align*}

Let us denote by $\cH_t$ the history $\{\w_\tau,U_\tau,\o_\tau\}_{\tau = 1}^{t-1} \cup \w_{t}$ till time $t$. Then conditioning on the history $\cH_t$ till time $t$, and taking expectation over $U_t:=\{\u_t^1,\ldots,\u_t^m\}$ we further get:

\begin{align*}
& {m^2}\E_{U_t}[\|\w_{t+1} - \x^*\|^2 \mid \cH_t]  \\
& = m(\norm{\w_{t}-\x^*}^2 + \eta^2) - 2\eta \sum_{i = 1}^{m-1}\sum_{j = i+1}^m \bign{\w_t-\x^*}^\top\E_{U_t}[(\g_t^i + \g_t^j) \mid \cH_t] + 2\frac{m(m-1)}{2}\norm{\w_{t}-\x^*}^2  
\\
& \hspace{3in} - 2\eta\sum_{i = 1}^m(\w_t - \x^*)^\top\g_t^i + 2\eta^2 \sum_{i = 1}^{m-1}\sum_{j = i+1}^m \E_{U_t}[{\g_t^i}^\top\g_t^j \mid \cH_t].
\end{align*}

Let us denote by $n_t = \frac{\nabla f(\w_t)}{\norm{\nabla f(\w_t)}}$ the normalized gradient at point $\w_t$. 
Now note for any $i \neq j$, since $\u_t^i$ and $\u_t^j$ are independent, from \cref{thm:normgrad} we get: 
\begin{align*}
    \E_{U_t}&[{\g_t^i}^\top\g_t^j \mid \cH_t] = \E_{\u_i}[ \E_{\u_j}[{\g_t^i}^\top\g_t^j \mid \u_t^i ]\mid \cH_t]
    \\
    & \leq \E_{\u_i}\bigg[ {\g_t^i}^\top \frac{n_t}{\sqrt d} + 2\lambda_t \mid \cH_t \bigg] 
    \\
    & \leq \frac{1}{\sqrt d}\biggn{\frac{n_t^\top n_t}{\sqrt d}} + 4\lambda_t = \frac{1}{d} + 4\lambda_t.
\end{align*}

where recall from \cref{thm:normgrad} and \cref{lem:biasgrad}, $\lambda_t \leq \frac{\beta \gamma \sqrt{d}}{\norm{\nabla f(\x)}}\bigg( 1 + 2\sqrt{\log \frac{\norm{\nabla f(\x)}}{\sqrt d \beta \gamma}} \bigg)$. Combining this with the main equation, and further applying \cref{thm:normgrad}, we get:

\begin{align*}
& {m^2}\E_{U_t}[\|\w_{t+1} - \x^*\|^2 \mid \cH_t]  \\
& \leq m(\norm{\w_{t}-\x^*}^2 + \eta^2) - 2\eta \sum_{i = 1}^{m-1}\sum_{j = i+1}^m \bign{\w_t-\x^*}^\top\E_{U_t}[(\g_t^i + \g_t^j) \mid \cH_t] + 2\frac{m(m-1)}{2}\norm{\w_{t}-\x^*}^2  
\\
& \hspace{2in} - 2\eta\sum_{i = 1}^m(\w_t - \x^*)^\top\g_t^i + 2\eta^2 \frac{m(m-1)}{2} \biggn{\frac{1}{d} + 4\lambda_t},
\\
& \leq m(\norm{\w_{t}-\x^*}^2 + \eta^2) - 2m\eta \norm{\w_t - \x^*}\sum_{i = 1}^{m} \frac{\bign{\w_t-\x^*}^\top}{\norm {\w_t-\x^*}}\E_{U_t}[\g_t^i \mid \cH_t] + 2\frac{m(m-1)}{2}\norm{\w_{t}-\x^*}^2  
\\
& \hspace{2in} + 2\eta^2 \frac{m(m-1)}{2} \biggn{\frac{1}{d} + 4\lambda_t},
\\
& = m(\norm{\w_{t}-\x^*}^2 + \eta^2) - 2m\eta \norm{\w_t - \x^*}\sum_{i = 1}^{m} [\frac{(\w_t-\x^*)^\top}{\norm{\w_t - \x^*}}\n_t + 2\lambda_t] + 2\frac{m(m-1)}{2}\norm{\w_{t}-\x^*}^2 
\\
& \hspace{2in} + 2\eta^2 \frac{m(m-1)}{2} \biggn{\frac{1}{d} + 4\lambda_t},
\\
& = m(\norm{\w_{t}-\x^*}^2 + \eta^2) - 2m\eta \norm{\w_t - \x^*}\sum_{i = 1}^{m} [\frac{(\w_t-\x^*)^\top}{\norm{\w_t - \x^*}}\n_t + 2\lambda_t] + 2\frac{m(m-1)}{2}\norm{\w_{t}-\x^*}^2 
\\
& \hspace{2in} + 2\eta^2 \frac{m(m-1)}{2} \biggn{\frac{1}{d} + 4\lambda_t},
\\
& = m^2\norm{\w_{t}-\x^*}^2 + \eta^2(m + \frac{m(m-1)}{d}) - 2\eta m^2[(\w_t-\x^*)^\top\n_t 
\\
& \hspace{3.1in} + 4\eta^2m^2\lambda_t + 4m^2\eta\norm{\w_t - \x^*}\lambda_t,
\\
& = m^2\norm{\w_{t}-\x^*}^2 + \eta^2(m + \frac{m(m-1)}{d}) - 2\eta m^2[(\w_t-\x^*)^\top\n_t 
\\
& \hspace{3.1in} + 4m^2\sqrt{d}\eta\norm{\w_t - \x^*} \lambda_t + 4m^2\sqrt{d}\eta\norm{\w_t - \x^*}\lambda_t
\end{align*}

where the last inequality follows by a choice of $\eta$ such that $\frac{\eta}{\norm{\w_t - \x^*}\sqrt d} \leq 1$ (will see shortly below why this is true). 
Further from Claim-2 of \cref{lem:cvxgrad}, and from the fact that $m<d$, we can derive:
\begin{align*}
m^2\E_{\u_t}[\|\w_{t+1} - \x^*\|^2 \mid \cH_t] 
 & \le m^2\|\w_{t} - \x^*\|^2 +  m^2(- 2\eta\frac{c\sqrt{2\epsilon}}{\sqrt {d\beta}} + 8\eta\lambda_t \sqrt{d}\|\w_{t} - \x^*\|)  + 2m\eta^2.
\end{align*}

Now, similar to the derivation followed in \cite{SKM21} (see proof of Lem 6, \cite{SKM21}), choosing $\gamma \le \frac{\|\nabf(\w_{t})\|}{960 \beta d\sqrt{d} \|\w_{t} - \x^*\| \sqrt{\log 480}}\sqrt{\frac{2 \epsilon}{\beta}}$, we can get:

\begin{align*}
\E_{\cH_t}[\E_{\u_t}[\|\w_{t+1} - \x^*\|^2] \mid \cH_t] &\le \|\w_{t} - \x^*\|^2 - \frac{\eta\sqrt{2\epsilon}}{10\sqrt{d \beta}} + \frac{\eta\sqrt{2\epsilon}}{20\sqrt{d \beta}} + \frac{2\eta^2}{m}.
\end{align*}		

One possible choice of $\gamma$ is $\gamma = \frac{ \epsilon^{3/2}}{960 \beta d\sqrt{d} D^2\ \sqrt{\log 480}}\sqrt{\frac{2}{\beta}}$ (since $\norm{\nabla f(\x)} \geq \frac{\epsilon}{D}$ for any $\x$ s.t. $f(\x)-f(\x^*) > \epsilon$ by \cref{lem:gradf_lb}). Then following from the above equation, we further get: 

\begin{align*}
& \E_{\cH_t}[\E_{\u_t}[\|\w_{t+1} - \x^*\|^2 \mid \cH_t]] \le \|\w_{t} - \x^*\|^2 -  \eta \frac{\sqrt{2\epsilon}}{20\sqrt{d \beta}} + \frac{2\eta^2}{m},\\
		& ~~~~~~~~~~~~~ = \|\w_{t} - \x^*\|^2  - \frac{ (\sqrt 2 -1)m\epsilon }{400 d\beta},~~~\Big(\text{setting } \eta = \frac{m\sqrt{\epsilon}}{20\sqrt{ d \beta}}\Big)\\
		\overset{}{\implies} & \E_{\cH_T}[\|\w_{T+1} - \x^*\|^2] \le \|\w_{1} - \x^*\|^2  - \frac{  (\sqrt 2 -1)m\epsilon T}{400 d\beta},
		\\
		& \hspace{1.8in} \big(\text{summing over} t = 1, \ldots T \text{ and laws of iterated expectation}\big).
\end{align*}

Above implies, if indeed $f(\w_{\tau}) - f(\x^*) > \epsilon$ continues to hold for all $\tau = 1,2, \ldots T$, then $\E[\|\w_{T+1} - \x^*\|^2] \le 0$, for $T \ge  \frac{400 md\beta}{(\sqrt 2 -1)\epsilon}(\|\w_{1} - \x^*\|^2)$, which basically implies $\w_{T+1} = \x^*$ (i.e. $f(\w_{T+1}) = f(\x^*)$). Otherwise there must have been a time $t \in [T]$ such that $f(\w_{t}) - f(\x^*) < \epsilon$. 
 
 The last bit of the proof lies in ensuring that indeed $\frac{\eta}{\norm{\w_t - \x^*}\sqrt d} \leq 1$ in all those rounds where $f(\w_{t+1})-f(\x^*) > \epsilon$. This is easy to note given $\beta$-smoothness as: 
 \begin{align*}
 \frac{\eta}{\norm{\w_t - \x^*}\sqrt d} \leq \frac{m \sqrt \epsilon}{20 d\sqrt{\beta}\norm{\w_t - \x^*}} \leq \frac{m \sqrt{f(\w_t)-f(\x^*)}}{20 d\sqrt{ \beta}\norm{\w_t - \x^*}} \leq \frac{m \sqrt{\beta}\norm{\w_t - \x^*}}{20 d\sqrt{ \beta}\norm{\w_t - \x^*}} \leq \frac{m}{20 d} < 1.   
 \end{align*}
 This concludes the proof with $T_\epsilon = T$, which gives an $O(m)$-factor improvement over the convergence bounds with single-$\sign$ feedback (as derived in \cref{thm:single}). 
\end{proof}

\subsection{Proof of \cref{thm:batch_strng}}
\label{app:batch_strng}

\thmbatchstrng*

\begin{proof}[Proof of \cref{thm:batch_strng}]
Let $\cH_k:= \{\w_{k'},(\w_{t'},U_{t'},o_{t'})_{t' \in t_{k'}}\}_{k' = 0}^{k} \cup \{\w_{k+1}\}$ denotes the complete history till the end of phase $k$ for all $k \in [k_{\epsilon}]$. 
By \cref{thm:batch_strng} we know that, for any fixed $T>0$, when  \cref{alg:batch} is run with
$\eta = \frac{m\sqrt{\epsilon}}{20\sqrt{ d \beta}}$,  $ \gamma =  \frac{ \epsilon^{3/2}}{960 \beta d\sqrt d D^2\ \sqrt{\log 480}}\sqrt{\frac{2}{\beta}}$ and $\epsilon = \frac{400 d\beta D}{(\sqrt 2 -1)T}$ $(D:= \|\w_{1} - \w^*\|^2)$, \cref{alg:batch} returns 
\[
\E[f(\m_{T+1})] - f(\x^*) \le \epsilon = \frac{400 d\beta \|\w_{1} - \x^*\|^2}{(\sqrt 2 -1)T}
\] 
with sample complexity (number of pairwise comparisons) $2T$. 

However, in this case since $f$ is also $\alpha$-strongly convex \cref{lem:prop_alpha} further implies
\begin{align}
\label{eq:alpha1-a}
\E & [\nicefrac{\alpha}{2}\|\m_{T+1} - \x^*\|^2] \le \E[f(\m_{T+1})] - f(\x^*) \le \frac{400 d\beta \|\w_{1} - \x^*\|^2}{(\sqrt 2 -1)T}\\
\nonumber \implies & \E[\|\m_{T+1} - \x^*\|^2] \le \frac{800 d\beta \|\w_{1} - \x^*\|^2}{(\sqrt 2 -1)\alpha T}
\end{align}

Now initially for $k = 1$, clearly applying the above result for $T = t\|\w_{1} - \x^*\|^2$, we get 
\begin{align*}
	\E[\|\w_{2} - \x^*\|^2] \le \frac{800 d\beta \|\w_{1} - \x^*\|^2}{(\sqrt 2 -1)\alpha T} = 1
\end{align*}

Thus, for any $k = 2,3,\ldots k_{\epsilon}-1$, given the initial point $\w_k$, if we run \cref{alg:argmin} 
with $T = 2t = \frac{1600 d \beta}{(\sqrt 2 -1)\alpha}$, we get from \eqref{eq:alpha1-a} 

\begin{align*}
\E_{\cH_k}[\|\w_{k+1} - \x^*\|^2 \mid \cH_{k-1}] \le \frac{800 d\beta \|\w_{k} - \x^*\|^2}{(\sqrt 2 -1)\alpha T} = \frac{\|\w_{k} - \x^*\|^2}{2}
\end{align*}

This implies given the history till phase $k-1$, using \cref{eq:alpha1-a} and our choice of $t_k$, 
\begin{align*}
	\E_{\cH_k}[f(
	\w_{k+1}) - f(\x^*) \mid \cH_{k-1}]  \le  \E_{\cH_k}[\frac{1}{4\alpha}\|\w_{k} - \x^*\|^2 \mid \cH_{k-1}] & \le  \frac{1}{4\alpha}(\frac{1}{2})^{k-1} \|\w_1 - \x^*\|^2 \le \frac{1}{\alpha 2^{k+1}}.
\end{align*}

Thus, to ensure at $k=k_\epsilon$, $\E[f(\w_{k_\epsilon +1}) - f(\x^*)] \le \epsilon$, this demands  
$(\nicefrac{1}{2})^{k_\epsilon+1}\alpha \le \epsilon$, or equivalently $\frac{\alpha}{2 \epsilon} \le 2^{k_\epsilon+1}$, which justifies the choice of $k_\epsilon = \log_2\big(\frac{\alpha}{\epsilon}\big) $. By \cref{thm:batch}, recall running the subroutine B-NGD$(\w_{k},\eta_k, \gamma_k,t_k,m)$ actually requires a query complexity of $2t_k = 4t$, and hence the total query complexity (over $k_\epsilon$ phases) of \cref{alg:argmin} becomes $4t k_\epsilon + t_1 = O\Big( \frac{800 d \beta}{(\sqrt 2 -1)\alpha}(\log_2\big(\frac{\alpha}{\epsilon}\big) + D) \Big)$, where recall $D:=\|\w_1 - \x^*\|^2$.	
\end{proof}	

\begin{restatable}[\citep{hazanbook,bubeckbook}]{lem}{propalpha}
	\label{lem:prop_alpha}
	If $f: \cD \mapsto \R$ is an $\alpha$-strongly convex function, with $\x^*$ being the minimizer of $f$. Then for any $\x \in \cR$,
	$
	\frac{\alpha}{2}\|\x^* - \x\|^2 \le f(\x) - f(\x^*).	
	$
\end{restatable}	

\begin{proof}
	This simply follows by the properties of $\alpha$-strongly convex function. Note by definition of $\alpha$-strong convexity, for any $\x,\y \in \R$,
\[
		f(\x) - f(\y) \ge \nabla f(\y)^\top (\x - \y) +  \frac{\alpha}{2}\|\x - \y\|^2. 
\]
The proof simply follows setting $\y = \x^*$.
\end{proof}

\section{Appendix for \cref{sec:argmin}}
\label{app:argmin}

\subsection{Proof of \cref{thm:argmin}}
\label{app:thm_argmin}

\thmargmin*

\begin{proof}[Proof of \cref{thm:argmin}]
We start by noting that at each round $t$, the the algorithm receives the winner feedback $o_t = \arg\min(f(\x_t^1), f(\x_t^2), \ldots, f(\x_t^{\tilde m}))$, upon playing the subset $S_t = \{\x_t^1,\x_t^2, \ldots, \x_t^{\tilde m}\}$, we recall $\tilde m = 2^{\ell_m}$. 

Moreover by definition, $f(\x_t^{o_t}) < f(\x_t^i), ~\forall i \neq o^t, i \in [\ell_m]$, and hence 
$\sign\bign{f(\x_t^{o_t})-f(\x_t^i)} = -1$ for all $i$. 
%
Now let us denote by $\y_t^i = \w_t + \gamma U_t \v'_i$, where $v'_i = - v_i$ and $v'_j = v_j, \forall j \in [\ell_m]$.
Note $\v'_i \in \cN(\v,\cG)$, i.e. $\v'_i$ is a neighboring node of $\v$ in the graph $\cG(V_{\ell_m})$, and also $\y_t^i \in S_t$ by construction. Hence
\[
\sign\bign{ f(\x_t^{o_t}) - f(\y_t^i) } = -1.
\]
Combining the above and the definition of $\g_t^i$, this actually implies
\[
\sign\bign{ f(\x_t^{o_t}) - f(\y_t^i) }\u_t^i = v_i \u_t^i = g_t^i.
\]
But then by \cref{thm:new_normgrad} we have, for any $d$-dimensional unit norm vector $\b \in \S_d(1)$:
\begin{align}
\label{eq:argmin1}
    \frac{c}{\sqrt{d}} \frac{\nabla f(\w)^\top}{\norm{\nabla f(\w)}}\b  - 2\lambda_t
    \leq 
    \E_{\u_t^i}[- v_i\u_t^{i\top} \b] 
    \leq \frac{c}{\sqrt{d}} \frac{\nabla f(\w)^\top}{\norm{\nabla f(\w)}}\b  + 	2\lambda_t
,
\end{align}
$\lambda_t \leq \frac{\beta \gamma \sqrt{d \ell_m}}{\norm{\nabla f(\w_t)}}\bigg( 1 + 2\sqrt{\log \frac{\norm{\nabla f(\w_t)}}{ \beta \gamma \sqrt{d\ell_m}}} \bigg)$.	
   
Now note, by the update rule:
\begin{align*}
\norm{\w_{t+1} - \x^*}^2 \leq \norm{\tilde \w_{t+1} - \x^*}^2  =  \norm{\w_{t} -\frac{\eta}{\ell_m}\sum_{i = 1}^{\ell_m}\g_t^i - \x^*}^2, \end{align*}
since we had $\g_t:= \frac{1}{\ell_m}\sum_{i=1}^{\ell_m}\g_t^i$ and the first inequality holds since projection reduces the distance to optimal $\x^*$. 
Note the update of $\w_{t+1}$ appears to be the same update by $\ell_m$-batched sign feedback. Then following the same derivation of \cref{thm:batch} (as it proceeds from \cref{eq:thm1prf_init} in the proof of \cref{thm:batch}), with $\eta = \frac{\ell_m\sqrt{\epsilon}}{20\sqrt{ d \beta}}$, and  $ \gamma = \frac{ \epsilon^{3/2}}{960 \beta d\ell_m\sqrt {d \ell_m} D^2\ \sqrt{\log 480}}\sqrt{\frac{2}{\beta}}$ yields the desired result. 
\end{proof}

\subsection{Key Lemmas to Prove \cref{thm:argmin}}
\label{app:lem_argmin}

\begin{thm}
\label{thm:new_normgrad}	
Let $f$ is $\beta$-smooth function. 
Assume $\u_t^1,\u_t^2,\ldots\u_t^{\ell_m} \overset{\text{iid}}{\sim} \text{Unif}(\cS_d(\frac{1}{\sqrt \ell_m}))$. Denote $U_t:= [\u_t^1,\ldots,\u_t^{\ell_m}] \in \R^{d \times \ell_m}$.  Let $\x = \w + \gamma U_t \v$ for any $\v \in V_{\ell_m}$, and $\y = \w + \gamma U_t \v'$, for any $\v' \in \cN(\v,\cG)$, i.e. $\v'$ is any neighboring node of $\v$ in the graph $\cG(V_{\ell_m})$.
In particular, let $v'_i = - v_i$ and $v'_j = v_j, \forall j \in [\ell_m]$. 
Then 
\begin{align*}
    \frac{c}{\sqrt{d}} \frac{\nabla f(\w)^\top}{\norm{\nabla f(\w)}}\b  - 	2\lambda
    \leq 
    \E_{\u_i}[\sign (f(\x)-f(\y)) \u_i^\top \b] 
    \leq \frac{c}{\sqrt{d}} \frac{\nabla f(\w)^\top}{\norm{\nabla f(\w)}}\b  + 	2\lambda
,
\end{align*}
for some universal constant $c \in [\tfrac{1}{20},1]$, $\lambda \leq \frac{\beta \gamma \sqrt{d \ell_m}}{\norm{\nabla f(\w)}}\bigg( 1 + 2\sqrt{\log \frac{\norm{\nabla f(\w)}}{ \beta \gamma \sqrt{d\ell_m}}} \bigg)$ and $\b \in S_{d}(1)$ being any unit vector of dimension $d$. 
\end{thm}

The proof of \cref{thm:new_normgrad} follows combining the results of \cref{lem:newbiasgrad} and \cref{lem:normgrad} (from \citep{SKM21}), proved below.
 
\begin{restatable}[]{lem}{biasgrad}
	\label{lem:newbiasgrad}	
Let $f$ is $\beta$-smooth function. 
Assume $\u_t^1,\u_t^2,\ldots\u_t^{\ell_m} \overset{\text{iid}}{\sim} \text{Unif}(\cS_d(\frac{1}{\sqrt \ell_m}))$. Denote $U_t:= [\u_t^1,\ldots,\u_t^{\ell_m}] \in \R^{d \times \ell_m}$.  Let $\x = \w + \gamma U_t \v$ for any $\v \in V_{\ell_m}$, and $\y = \w + \gamma U_t \v'$, for any $\v' \in \cN(\v,\cG)$, i.e. $\v'$ is any neighboring node of $\v$ in the graph $\cG(V_{\ell_m})$.
In particular, let $v'_i = - v_i$ and $v'_j = v_j, \forall j \in [\ell_m]$. 
Then for any unit dimension $d$ vector $\b \in \cS_d(1)$ we have:
	\begin{align*}
		\Big|\E_{\u_i}[\sign  (f(\x)-f(\y)) \u_i^\top \b] - \E_{\u_i}[\sign(\nabla f(\w) \dotp \u_i) \u_i^\top \b] \Big|
		\leq 
		2\lambda,	
	\end{align*}
where $\lambda \leq \frac{\beta \gamma \sqrt{d \ell_m}}{\norm{\nabla f(\w)}}\bigg( 1 + 2\sqrt{\log \frac{\norm{\nabla f(\w)}}{ \beta \gamma \sqrt{d\ell_m}}} \bigg)$. 	
\end{restatable}

\begin{proof}
Without loss of generality, assume $\v = (1,1,\ldots,1) \in \{0,1\}^{\ell_m}$ and $\v' = (-1,1,\ldots,1)$, i.e. $i=1$. 
Thus $\x = \w + \gamma(\u_1 + \u_2 + \ldots + \u_{\ell_m})$, 
and $\y = \w + \gamma(-\u_1 + \u_2 + \ldots + \u_{\ell_m})$. 
Also let us denote by 
$\u = (\u_1 + \u_2 + \ldots + \u_{\ell_m})$, 
$\u' = (- \u_1 + \u_2 + \ldots + \u_{\ell_m})$. 

From smoothness we have
\begin{align*}
    \gamma \u_i \dotp \nabla f(\w) - \tfrac12 \beta \gamma^2
    &\leq
    f(\w+ \gamma \u) - f(\w) 
    \leq
     \gamma \u \dotp \nabla f(\w) + \tfrac12 \beta \gamma^2;
    \\
   \gamma \u' \dotp \nabla f(\w) - \tfrac12 \beta \gamma^2
    &\leq
    f(\w+ \gamma \u') - f(\w) 
    \leq
    \gamma \u' \dotp \nabla f(\w) + \tfrac12 \beta \gamma^2
    .
\end{align*}
Subtracting the inequalities, we get
\begin{align*}
    \abs*{ f(\w+ \gamma \u) - f(\w + \gamma \u') - 2 \gamma \u_1 \dotp \nabla f(\x) }
    \leq
    \beta \gamma^2
    .
\end{align*}
Therefore, if $\beta \gamma^2 \leq  \gamma \abs{ \u_1 \dotp \nabla f(\w) }$, we will have that $\sign(f(\w+ \gamma \u) - f(\w+ \gamma \u')) = \sign(\u_1 \dotp \nabla f(\w))$. 
Let us analyse $\Pr_{\u_1}(\beta \gamma \geq \abs{ \u_1 \dotp \nabla f(\w) })$.
We know for $\v \sim \cN(\0_d,\cI_d)$, $\tilde \v:= \v/\|\v\| \sim \cS_d(1)$, i.e. $\tilde \v$ is uniformly distributed on the unit sphere, and hence $\dfrac{\v}{\norm{\v}\sqrt {\ell_m}} \sim \cS_d(1/\sqrt{\ell_m})$.  Then can write:
\begin{align*}
    \P_{\u_1}&\brk!{ |\u_1 \dotp \nabla f(\w)| \leq \beta \gamma } 
    =
    \P_{\v}\brk!{ \abs{\v \dotp \nabla f(\w)} \leq \beta \gamma \norm{\v}\sqrt{\ell_m}}
    \\
    &\leq
    \P_{\v}\brk!{ \abs{\v \dotp \nabla f(\w)} \leq 2\beta \gamma \sqrt{d{\ell_m}\log (1/ \gamma')}} + \P_\v \brk{\norm{\v} \geq 2\sqrt{d\ell_m\log(1/ \gamma')}}
    \\
    &\leq 
    \P_{\v}\brk!{ \abs{\v \dotp \nabla f(\w)} \leq 2\beta \gamma \sqrt{d\ell_m \log( 1/ \gamma')}} + \gamma'
    ,
\end{align*}
where the final inequality is since $\P_\v \brk{\norm{\v}^2 \leq 2d\ell_m\log(1/ \gamma')} \geq 1- \gamma'$ for any $ \gamma'$ (see Lemma~\ref{lem:chi}).
On the other hand, since $\v \dotp \nabla f(\w) \sim \cN(0,\|\nabla f(\w)\|^2)$, we have for any $a>0$ that
\begin{align*}
    \Pr\brk*{ |\v \dotp \nabla f(\w)| \leq a }
    \leq
    \frac{2a}{\|\nabla f(\w)\| \sqrt{2\pi}}
    \leq
    \frac{a}{\|\nabla f(\w)\|}
    .
\end{align*}
Setting, $a = 2\beta \gamma \sqrt{d\ell_m \log( 1/ \gamma')}$, and combining the inequalities, we have that $\sign(f(\w+ \gamma \u) - f(\w-  \gamma \u)) = \sign(\u \dotp \nabla f(\w))$ except with probability at most
\begin{align*}
    \inf_{ \gamma'>0} \bigg\{ \gamma' + \frac{2\beta \gamma \sqrt{d\ell_m \log( 1/ \gamma')}}{\|\nabla f(\w)\|}\bigg\} = \lambda ~\text{ (say)}
   ,
\end{align*}
and further choosing $\gamma' = \frac{\beta \gamma \sqrt{d \ell_m}}{\norm{\nabla f(\w)}}$, we get that $\lambda \leq \frac{\beta \gamma \sqrt{d \ell_m}}{\norm{\nabla f(\w)}}\bigg( 1 + 2\sqrt{\log \frac{\norm{\nabla f(\w)}}{ \beta \gamma \sqrt {d\ell_m}}} \bigg)$. 
As for the claim about the expectation, note that for any vector $\b \in \cS_d(1)$,
\begin{align*}
\Big|\E_{\u_1}&[\sign(f(\w+ \gamma \u)-f(\w + \gamma \u')) \u_1\tr \b] - \E_{\u_1}[\sign(\nabla f(\w) \dotp \u_1) \u_1\tr \b] \Big|
\leq 
2\lambda
,
\end{align*}
as with probability $1-\lambda$ the two expectations are equal, and otherwise, they  differ by at most $2$.
\end{proof}

\begin{restatable}[\cite{SKM21}]{lem}{lemngrad}
\label{lem:normgrad}	
For a given vector $\g \in \R^d$ and a random unit vector $\u$ drawn uniformly from $\cS_d(1)$, we have
\begin{align*}
    \E[\sign(\g \dotp \u) \u]
    =
    \frac{c}{\sqrt{d}} \frac{\g}{\norm{\g}}
    ,
\end{align*}
for some universal constant $c \in [\tfrac{1}{20},1]$. 
\end{restatable}

\begin{proof}[This proof is same as the proof in \cite{SKM21}]
Without loss of generality we can assume $\norm{\g} = 1$, since one can divide by $\norm{\g}$ in both side of Lem. \ref{lem:normgrad} without affecting the claim.
Now to bound $\E[\abs{\g \dotp \u}]$, note that since $\u$ is drawn uniformly from $\cS_d(1)$, by rotation invariance this equals $\E[\abs{u_1}]$.
For an upper bound, observe that by symmetry $\E[u_1^2] = \tfrac1d \E[\sum_{i=1}^d u_i^2] = \tfrac1d$ and thus
\begin{align*}
    \E[\abs{u_1}]
    \leq
    \sqrt{\E[u_1^2]}
    =
    \frac{1}{\sqrt{d}}
    .
\end{align*}
We turn to prove a lower bound on $\E[\abs{\g \dotp \u}]$.
If $\u$ were a Gaussian random vector with i.i.d.~entries $u_i \sim \cN(0,1/d)$, then from standard properties of the (truncated) Gaussian distribution we would have gotten that $\E[\abs{u_1}] = \sqrt{2/\pi d}$. 
For $\u$ uniformly distributed on the unit sphere, $u_i$ is distributed as $v_1/\norm{\v}$ where $\v$ is Gaussian with i.i.d.~entries $\cN(0,1/d)$.
We then can write
\begin{align*}
    \Pr\brk*{ \abs{u_1} \geq \frac{\epsilon}{\sqrt{d}} }
    & =
    \Pr\brk*{ \frac{\abs{v_1}}{\norm{\v}} \geq \frac{\epsilon}{\sqrt{d}} }
    \geq
    \Pr\brk*{ \abs{v_1} \geq \frac{1}{\sqrt{d}} \mbox{ and } \norm{\v} \leq \frac{1}{\epsilon} }
    \\
    & \geq
    1 - \Pr\brk*{ \abs{v_1} < \frac{1}{\sqrt{d}}} - \Pr\brk*{\norm{\v} > \frac{1}{\epsilon} }
    .
\end{align*}
Since $\sqrt{d} v_1$ is a standard Normal, we have
\begin{align*}
    \Pr\brk*{ \abs{v_1} < \frac{1}{\sqrt{d}}}
    =
    \Pr\brk*{ -1 < \sqrt{d} v_1 < 1}
    =
    2\Phi(1) - 1
    \leq
    0.7
    ,
\end{align*}
and since $\E[\norm{\v}^2] = 1$ an application of Markov's inequality gives
\begin{align*}
    \Pr\brk!{\norm{\v} > \frac{1}{\epsilon}}
    =
    \Pr\brk!{\norm{\v}^2 > \frac{1}{\epsilon^2}}
    \leq
    \epsilon^2 \E[\norm{\v}^2]
    =
    \epsilon^2
    .
\end{align*}
For $\epsilon=\tfrac14$ this implies that $\Pr\brk!{ \abs{u_1} \geq 1/4\sqrt{d} } \geq \tfrac15$, whence
$
    \E[\abs{\g \dotp \u}]
    =
    \E[\abs{u_1}]
    \geq
    1/20\sqrt{d}
    .
$
\end{proof}


\subsection{Proof of \cref{thm:argmin_strng}}
\label{app:thm_argmin_strng}

\vspace{-5pt}
\begin{center}
\begin{algorithm}[h]
\caption{\textbf{Improved Battling-NGD} (with Strong Convexity)}
\label{alg:armin_strng}
\begin{algorithmic}[1] 
\STATE {\bfseries Input:} Error tolerance $\epsilon > 0$, Batch size $m$%
\STATE {\bf Initialize} Initial point: $\w_1 \in \R^d$ such that $D := \|\w_1 - \x^*\|^2$ (assume known). 
\\ Phase counts $k_\epsilon:= \lceil \log_2\big(\frac{\alpha}{\epsilon}\big)  \rceil$, $t \leftarrow \frac{800 d \beta}{(\sqrt 2 -1)\alpha}$   
\\ $\eta_1 \leftarrow \frac{m\sqrt{\epsilon_1}}{20\sqrt{ d \beta}}, \epsilon_1 = \frac{400 d\beta D}{(\sqrt 2 -1)t_1} = 1$, $t_1 = t\|\w_1-\x^*\|^2$
\\ $ \gamma_1 \leftarrow \frac{ \epsilon_1^{3/2}}{960 \beta d\ell_m\sqrt{d\ell_m} D^2\ \sqrt{\log 480}}\sqrt{\frac{2}{\beta}}$.
\STATE Update $\w_{2} \leftarrow$ \textbf{Battling-NGD}$\big(\w_1,\eta_1, \gamma_1,t_1,m\big)$
\FOR{$k = 2,3,\ldots, k_\epsilon$}
\STATE $\eta_k \leftarrow \frac{m\sqrt{\epsilon_k}}{20\sqrt{ d \beta}}, \epsilon_k = \frac{400 d\beta}{(\sqrt 2 -1)t_k}$, $t_k = 2t$
\\$ \gamma_k \leftarrow \frac{ \epsilon_k^{3/2}}{960 \beta d\ell_m\sqrt{d\ell_m} D^2\ \sqrt{\log 480}}\sqrt{\frac{2}{\beta}}$.
\STATE Update $\w_{k+1} \leftarrow$ \textbf{Battling-NGD}$\big(\w_{k},\eta_k, \gamma_k,t_k,m\big)$
\ENDFOR   
\STATE Return $\m_{\epsilon} = \w_{k_\epsilon+1}$
\end{algorithmic}
\end{algorithm}
\end{center}
\vspace{-5pt}

\thmargstrng*

\begin{proof}[Proof of \cref{thm:argmin_strng}]
The proof follows from the exactly same analysis as the proof of \cref{thm:batch_strng}.  
\end{proof}

\section{Projected Dueling Convex Optimization with Single Sign Feedback} 
\label{sec:single}


\textbf{Main Idea: Estimating Gradient Directions (Normalized Gradients): } The algorithmic idea of our proposed algorithm is almost the same as what was proposed in \cite{SKM21}. Essentially, we start with any arbitrary `current estimate' of the function minimizer $\w_1 \in \cD$, and at any round $t$, we compute normalized gradient estimate $\g_t$ such that $\g_t:=o_t\u_t$, $\u_t \sim \text{Unif}(\cS_d(1))$ being any random unit direction in $\R^d$ and $o_t = \sign\Bign{f(\x_t)-f(\y_t)}$ is the sign feedback of the queried duel $(x_t,y_t)$ at round $t$, such that 
$\x_t = \w_t + \gamma\u_t$,
and
$\y_t = \w_t - \gamma\u_t$, $\gamma$ being any tunable perturbation step size.
Subsequently, the algorithm takes a step along the estimated descent direction $\g_t$ with (tunable) step-size $\eta$ and updates the current estimate $\w_{t+1} \leftarrow \w_t - \eta \g_t$ and repeat up to any given number of $T$ steps before outputting the final estimate of the minimizer $\w_{T+1}$. 

\subsection{Algorithm Design: Projected Normalized Gradient Descent (P-NGD)} 

In this section, we analyzed the \probsgn\, problem for a bounded decision space $\cD$. We describe a normalized gradient descent based algorithm for the purpose and anlyzed its convergence guarantee in \cref{thm:single}. 
It's important to note that the same problem was analyzed in \citep{SKM21}, but their analysis was limited to limited unbounded decision spaces only, which is unrealistic for practical problems and also led to more complication tuning of the learning parameters $\gamma>0$ and $\eta>0$.

\begin{center}
	\begin{algorithm}[h]
		\caption{Projected Normalized Gradient Descent (P-NGD)} 
		\label{alg:pngd}
		\begin{algorithmic}[1] 
			\STATE {\bfseries Input:} Initial point: $\w_1 \in \cD$, Initial distance: $D$ s.t. $D \geq \|\w_1 - \x^*\|^2$, Learning rate $\eta$, Perturbation parameter $ \gamma$, Query budget $T$ 
			\STATE {\bf Initialize} Current minimum $\m_1 = \w_1$
			\FOR{$t = 1,2,3,\ldots, T$}
			\STATE Sample $\u_t \sim \text{Unif}(\cS_d(1))$ 
			\STATE Set  $\x_{t} := \w_t + \gamma \u_t$,~ 
			            $\y_{t} := \w_t - \gamma \u_t$
			\STATE Play the duel $(\x_{t},\y_t)$, and observe $o_t \in {\pm 1}$ such that $o_t = \sign\big( f(\x_{t}) - f(\y_{t}) \big)$. 
			\STATE Update $\tw_{t+1} \leftarrow \w_t - \eta \g_t$, 
			where $\g_t = o_t \u_t$
			\STATE Project $\w_{t+1} = \arg\min_{\w \in \cD}\norm{\w - \tw_{t+1}}$
                \STATE Query the pair $(\m_{t}, \w_{t+1})$ and receive $\sign\bign{ f(\m_{t}) - f(\w_{t+1}) }$. 
			\STATE Update $\m_{t+1} \leftarrow 
			\begin{cases}
				\m_{t} ~\text{ if } \sign\bign{ f(\m_{t}) - f(\w_{t+1}) }<0\\
				\w_{t+1} ~\text{ otherwise }  
			\end{cases}$
			\ENDFOR   
			\STATE Return $\m_{T+1}$
		\end{algorithmic}
	\end{algorithm}
\end{center}

\textbf{Algorithm description: P-NGD } Our algorithm follows the same strategy same as the $\beta$-NGD (Algorithm 1) of \citep{SKM21} modulo a projection step (see Line $8$) which we had to incorporate for assuming bounded decision space $\cD$: The main idea is to estimate gradient direction of any point $\w \in \cD$ (normalized gradient estimate) querying the $\sign$ feedback of two symmetrically opposite points $(\w + \gamma\u, \w-\gamma\u)$, $\u \sim \text{Unif}(\cS_d(1))$ being any random unit direction, and simply take a `small-enough' step $(\eta)$ in the opposite direction of the estimated gradient. More formally, at each round $t$ the algorithm maintains a current point $\w_t$, initialized to any random point $\w_1 \in \cD$, and query two symmetrically opposite points $(\w_t + \gamma\u_t, \w_t-\gamma\u_t)$ along a random unit direction $\u_t\sim \text{Unif}(\cS_d(1))$. Following this it finds a normalized gradient estimate at $\w_t$, precisely $g_t = o_t\u_t$ based on \cref{thm:normgrad}, where $o_t = \sign\bign{f(\w_t + \gamma\u_t)-f(\w_t + \gamma\u_t)}$ is the sign feedback of the queried duel $(x_t,y_t)$. Subsequently, we update the running prediction using a (normalized) gradient descent step: $\tilde \w_{t+1} \leftarrow \w_t - \eta \g_t$ followed by a projection $\w_{t+1} = \arg\min_{\w \in \cD}\norm{\w - \tw_{t+1}}$.  The algorithm also maintains a running minimum $\m_t$ which essentially keeps track of $\min\{\w_1,\ldots,\w_t\}$. The complete algorithm is given in \cref{alg:pngd}.

\begin{restatable}[Convergence Analysis of \cref{alg:pngd} for $\beta$-Smooth Functions]{thm}{thmsingle}
\label{thm:single}
Consider $f$ to be $\beta$ smooth, and the desired accuracy level (suboptimality gap) is given to be $\epsilon > 0$.   
Then if Alg. \ref{alg:pngd} is run with $\eta = \frac{\sqrt{\epsilon}}{20\sqrt{ d \beta}}$,  $ \gamma = \frac{ \epsilon^{3/2}}{480 \beta d D^2\ \sqrt{\log 480}}\sqrt{\frac{2}{\beta}}$ and $T = T_\epsilon =  O\Big(\frac{d\beta D}{m \epsilon} \Big)$, $D\geq \|\w_{1} - \x^*\|^2$ being any upper bound on the initial distance from the optimal,  Alg. \ref{alg:pngd} returns an $\epsilon$-optimal point in at most $2T_\epsilon$ pairwise queries; i.e.
\[
\E[f(\m_{T+1})] - f(\x^*) \le \epsilon.
\] 
\end{restatable}

\begin{proof}[Proof of \cref{thm:single}]
 The proof idea crucially relies on \cref{lem:sign1}, which essentially shows that if we start from an initial point $(\w_1)$, which is more than $\epsilon$-suboptimal, i.e. $f(\w_1) - f(\w^*) > \epsilon$, then we will have $\E[f(\m_{T+1})] - f(\x^*) \le \epsilon$. 

The formal statement is as follows: 
\begin{restatable}[]{lem}{lemsign}
\label{lem:sign1}
Consider $f$ is $\beta$ smooth. Then in Alg. \ref{alg:pngd}, if the initial point $\w_1$ is such that $f(\w_1) - f(\x^*) > \epsilon$ (for $\epsilon > 0$), and the tuning parameters $T$, $ \gamma$ and $\eta$ is as in defined in \cref{thm:single}, we will have $\E[f(\m_{T+1})] - f(\x^*) \le \epsilon$.
\end{restatable}

Given \cref{lem:sign1}, the statement of \cref{thm:single} follows straightforwardly: Note $\m_{t}$ essentially keeps track of $\min_{t \in [T]} \w_t$. Now either $\w_1$ is such that $f(\w_1) - f(\x^*) < \epsilon$, in case the bound of \cref{thm:single} is trivially true as by definition $f(\m_{T+1}) \leq f(\w_1)$. On the other hand, if $f(\w_1) - f(\x^*) > \epsilon$, bound of \cref{thm:single} follows by \cref{lem:sign1}. We discuss the proof of \cref{lem:sign1} below.
\end{proof}

\begin{proof}[Proof of \cref{lem:sign1}]
Our main claim lies in showing that at any round $t$, if the iterate is at least $\epsilon$ away from the optional, i.e. $f(\w_t) - f(\x^*) > \epsilon$, then on expectation $\w_{t+1}$ comes closer to $\x^*$ (in $\ell_2$ distance), compared to $\w_{t}$. The following argument proves it formally:

Consider any $t = 1,2,\ldots T$, such that $f(\w_t) > f(\x^*) + \epsilon$. Let us denote by $n_t = \frac{\nabf(\w_t)}{\|\nabf(\w_t)\|}$, the normalized gradient at point $\w_t$. Now from the update rule, we get that:

\begin{align*}
    \norm{\w_{t+1} - \x^*}^2 \le \norm{\tilde \w_{t+1} - \x^*}^2 \le \norm{\w_{t} - \x^*}^2 - 2 \eta \g_t^\top (\w_t - \x^* ) + \eta^2. 
\end{align*}

where the first inequality holds since projection reduces distance to optimal $\x^*$. Let us denote by $\cH_t$ the history $\{\w_\tau,\u_\tau,\m_\tau\}_{\tau = 1}^{t-1} \cup \w_{t}$ till time $t$. Then conditioning on the history $\cH_t$ till time $t$,  and taking expectation over $\u_t$ we further get:

\[
\E_{\u_t}[\|\w_{t+1} - \x^*\|^2 \mid \cH_t] \le \E_{\u_t}[\|\w_{t} - \x^*\|^2 \mid \cH_t] - 2 \eta \E_{\u_t}[\g_t^\top \mid \cH_t] (\w_t - \x^* ) + \eta^2,
\]
 Further applying \cref{thm:normgrad}, one can get:

\begin{align*}
\E_{\u_t}[\|\w_{t+1} - \x^*\|^2 \mid \cH_t] 
 & 
 \le \|\w_{t} - \x^*\|^2 - 2 \eta \bigg( \E_{\u_t}[\sign(\nabla f(\w_{t}) \dotp \u) \u^\top \frac{(\w_{t} - \x^*)}{\|\w_{t} - \x^*\|}] - 2\lambda_t \bigg) \|\w_{t} - \x^*\| + \eta^2, 
 \\
 & \le \|\w_{t} - \x^*\|^2 - 2\eta\frac{c}{\sqrt d}\n_t^\top(\w_{t} - \x^*) + 4\eta\lambda_t \|\w_{t} - \x^*\|  + \eta^2,
\end{align*}
where recall from \cref{thm:normgrad} and \cref{lem:biasgrad}, $\lambda_t \leq \frac{\beta \gamma \sqrt{d}}{\norm{\nabla f(\x)}}\bigg( 1 + 2\sqrt{\log \frac{\norm{\nabla f(\x)}}{\sqrt d \beta \gamma}} \bigg)$. Further from Claim-2 of \cref{lem:cvxgrad}, we have:
\begin{align*}
\E_{\u_t}[\|\w_{t+1} - \x^*\|^2 \mid \cH_t] 
& \le \|\w_{t} - \x^*\|^2 - 2\eta\frac{c\sqrt{2\epsilon}}{\sqrt {d\beta}} + 4\eta\lambda_t \|\w_{t} - \x^*\|  + \eta^2,
\end{align*}

Now, similar to the derivation followed in \cite{SKM21} (see proof of Lem 6, \cite{SKM21}), choosing $\gamma \le \frac{\|\nabf(\w_{t})\|}{480 \beta d \|\w_{t} - \x^*\| \sqrt{\log 480}}\sqrt{\frac{2 \epsilon}{\beta}}$, we can get:

\begin{align*}
\E_{\cH_t}[\E_{\u_t}[\|\w_{t+1} - \x^*\|^2] \mid \cH_t] &\le \E_{\cH_t}[\|\w_{t} - \x^*\|^2] - \frac{\eta\sqrt{2\epsilon}}{10\sqrt{d \beta}} + \frac{\eta\sqrt{2\epsilon}}{20\sqrt{d \beta}} + \eta^2,
\end{align*}		

so one possible choice of $\gamma$ is $\gamma = \frac{ \epsilon^{3/2}}{480 \beta d D^2\ \sqrt{\log 480}}\sqrt{\frac{2}{\beta}}$ (since $\norm{\nabla f(\x)} \geq \frac{\epsilon}{D}$ for any $\x$ s.t. $f(\x)-f(\x^*) > \epsilon$). Then following from the above equation, we further get: 

\begin{align*}
& \E_{\cH_t}[\E_{\u_t}[\|\w_{t+1} - \x^*\|^2 \mid \cH_t]] \le \E_{\cH_t}[\|\w_{t} - \x^*\|^2] -  \eta \frac{\sqrt{2\epsilon}}{20\sqrt{d \beta}} + \eta^2,\\
		& ~~~~~~~~~~ = \E_{\cH_t}[\|\w_{t} - \x^*\|^2]  - \frac{ (\sqrt 2 -1)\epsilon }{400 d\beta},~~~\Big(\text{since } \eta = \frac{\sqrt{\epsilon}}{20\sqrt{ d \beta}}\Big)\\
		\overset{}{\implies} & \E_{\cH_T}[\|\w_{T+1} - \x^*\|^2] \le \|\w_{1} - \x^*\|^2  - \frac{  (\sqrt 2 -1)\epsilon T}{400 d\beta}, ~\big(\text{summing } t = 1, \ldots T \text{ and laws of iterated expectation}\big)
\end{align*}
	
	Above implies, if indeed $f(\w_{\tau}) - f(\x^*) > \epsilon$ continues to hold for all $\tau = 1,2, \ldots T$, then $\E[\|\w_{T+1} - \x^*\|^2] \le 0$, for $T \ge  \frac{400 d\beta}{(\sqrt 2 -1)\epsilon}(\|\w_{1} - \x^*\|^2)$, which basically implies $\w_{T+1} = \x^*$ (i.e. $f(\w_{T+1}) = f(\x^*)$). Otherwise there must have been a time $t \in [T]$ such that $f(\w_{t}) - f(\x^*) < \epsilon$. This concludes the proof with $T_\epsilon = T$. 
\end{proof}

\section{Some Useful Results:}
\label{app:sign_feedback}

\begin{restatable}[]{lem}{cvxgrad}
\label{lem:cvxgrad}
Suppose $f: \cD \mapsto \R$ is a convex function such that $f(\y) < f(\x)$. Then $\Big(\frac{\nabla f(\x)}{\|\nabla f(\x)\|_2}\Big)^\top (\y - \x) \le 0$.
 Further if $\z$ is a point such that $f(\z) - f(\x^*) > \epsilon$, then one can show that $-\frac{\nabla f(\z)}{\norm{\nabla f(\z)} }^\top  (z - \x^*) \le - \sqrt{\frac{2\epsilon}{\beta}}$.
\end{restatable}

\begin{proof}
\textbf{Proof of Claim-1:}	To show the first part of the claim, note that since $f$ is convex,
	\[
	f(\y) \ge f(\x) + \nabla f(\x)^\top (\y - \x) \implies \nabla f(\x)^\top (\y - \x) \le f(\y) - f(\x) \le 0,
	\]
	which proves the claim by dividing both sides with $\|\nabla f(\x)\|_2$.

\textbf{Proof of Claim-2:}
Assume another point $\tilde \z:= \x^* + \sqrt{\frac{2\epsilon}{\beta}}\n_z$, where we denote by $n_{\z}:= \frac{\nabla f(\z)}{\norm{\nabla f(\z)}}$. 
	Now using $\beta$-smoothness of $f$ we have:
 $f(\tilde \z) \le f(\x^*) + \nabla f(\x^*)(\tilde \z - \x^*) + \frac{\beta}{2}\|\tilde \z - \x^*\|^2 = f(\x^*) + \epsilon$. Thus we have $f(\tilde \z) < f(\x^*) + \epsilon < f(\z)$, and hence from \cref{lem:cvxgrad}, we get $\n_\z^\top (\tilde \z - \z) \le 0$. But note this further implies $\n_\z^\top \bigg(\x^* +\sqrt{\frac{2\epsilon}{\beta}}\n_\z - \z \bigg) \le 0 \implies -\n_\z^\top  (\z - \x^*) \le - \sqrt{\frac{2\epsilon}{\beta}}$.
\end{proof}

\begin{restatable}[]{lem}{gradflb}
	\label{lem:gradf_lb}
	Suppose $f: \cD \mapsto \R$ is a convex function for some convex set $\cD \subseteq \R^d$ such that for any $\x,\y \in \R^d$, $f(\x) - f(\y) > \epsilon$. Then this implies $\| \nabla f(\x)\| > \frac{\epsilon}{\|\x - \y\|}$. 
	Further, assuming $D: = \max_{\x, \y \in \R^d}\|\x-\y\|_2$, we get $\| \nabla f(\x)\| > \frac{\epsilon}{D}$ for any $\x \in \cD$.
\end{restatable}

\begin{proof}
	The proof simply follows using convexity of $f$ as:
	\begin{align*}
		f(\x) - f(\y) > \epsilon & \implies \epsilon < f(\x) - f(\y) \le \nabla f(\x)(\x - \y) \le \| \nabla f(\x)\|_2 \|\x - \y\|_2 \\
		& \implies \| \nabla f(\x)\| \ge \frac{\epsilon}{\|\x - \y\|}.
	\end{align*}
	\vspace{-20pt}
\end{proof}

As shown in \cite{SKM21}, using the above result one can obtain the normalized gradient estimate of $f$ at any given point $\x$, as described below: 

\begin{thm}[Adapted from \cite{SKM21} with Slight Modifications]
	\label{thm:normgrad}	
	If $f$ is $\beta$-smooth, for any $\u \sim \text{Unif}(\cS_{d}(1))$, $\delta \in (0,1)$ and vector $\b \in \cS_d(1)$:
	\begin{align*}
		\E_{\u}[\sign (f(\x+\delta \u)-f(\x-\delta \u)) \u^\top \b] \le \frac{c}{\sqrt{d}} \frac{\nabla f(\x)^\top}{\norm{\nabla f(\x)}}\b  + 	2\lambda
	,
	\end{align*}
	for some universal constant $c \in [\tfrac{1}{20},1]$, and $\lambda \leq \frac{\beta \gamma \sqrt{d}}{\norm{\nabla f(\x)}}\bigg( 1 + 2\sqrt{\log \frac{\norm{\nabla f(\x)}}{\sqrt d \beta \gamma}} \bigg)$. 
\end{thm}

\textbf{Proof of \cref{thm:normgrad}.}
The proof mainly lies on the following lemma that shows how to the comparison feedback of two close points, $\x+ \gamma \u$ and $\x- \gamma \u$, can be used to recover a directional information of the gradient of $f$ at point $\x$. 
\begin{restatable}[]{lem}{biasgrad}
	\label{lem:biasgrad}	
	If $f$ is $\beta$-smooth, for any $\u \sim \text{Unif}(\cS_{d}(1))$, and $ \gamma \in (0,1)$, 
	then with probability at least $1-\lambda$ where $\lambda=\frac{\beta \gamma \sqrt{d}}{\norm{\nabla f(\x)}}\bigg( 1 + 2\sqrt{\log \frac{\norm{\nabla f(\x)}}{\sqrt d \beta \gamma}} \bigg)$, we have
	\begin{align*}
		\sign(f(\x+ \gamma \u)-f(\x- \gamma \u))\u =
		\sign(\nabla f(\x) \dotp \u) \u.
	\end{align*}
	
	Consequently, for any vector $\b \in \cS_d(1)$ we have
	$
		\Big|\E_{\u}[\sign  (f(\x+ \gamma \u)-f(\x- \gamma \u)) \u^\top \b] - \E_{\u}[\sign(\nabla f(\x) \dotp \u) \u^\top \b] \Big|
		\leq 
		2\lambda.
	$
\end{restatable}

\begin{rem}[Ensuring $\lambda$ denotes a valid probability]
It is important and assuring to note that when for any $\x \in \cD$ such that $f(\x) - f(\x^*) > \epsilon$ (which in turn implies $\norm{\nabla f(\x)} \geq \frac{\epsilon}{D}$ by \cref{lem:gradf_lb}), $\lambda \in [0,1]$ for any choice of $\gamma \in [0,\frac{\epsilon}{\beta D \sqrt d}]$ (Note we respect this in our choice of $\gamma$ for the algorithm guarantees, e.g. \cref{thm:single}, \cref{thm:batch}, etc).
\end{rem}

The result of Thm. \ref{thm:normgrad} now simply follows by combining the guarantees of Lem. \ref{lem:normgrad} and \ref{lem:biasgrad}.
$\hfill \square$

\begin{proof}[Proof of \cref{lem:biasgrad}]
From smoothness we have
\begin{align*}
     \gamma \u \dotp \nabla f(\x) - \tfrac12 \beta \gamma^2
    &\leq
    f(\x+ \gamma \u) - f(\x) 
    \leq
     \gamma \u \dotp \nabla f(\x) + \tfrac12 \beta \gamma^2
    ;
    \\
    - \gamma \u \dotp \nabla f(\x) - \tfrac12 \beta \gamma^2
    &\leq
    f(\x- \gamma \u) - f(\x) 
    \leq
    - \gamma \u \dotp \nabla f(\x) + \tfrac12 \beta \gamma^2
    .
\end{align*}
Subtracting the inequalities, we get
\begin{align*}
    \abs*{ f(\x+ \gamma \u) - f(\x-  \gamma \u) - 2 \gamma \u \dotp \nabla f(\x) }
    \leq
    \beta \gamma^2
    .
\end{align*}
Therefore, if $\beta \gamma^2 \leq  \gamma \abs{ \u \dotp \nabla f(\x) }$, we will have that $\sign(f(\x+ \gamma \u) - f(\x-  \gamma \u)) = \sign(\u \dotp \nabla f(\x))$. 
Let us analyse $\Pr_{\u}(\beta \gamma \geq \abs{ \u \dotp \nabla f(\x) })$.
We know for $\v \sim \cN(\0_d,\cI_d)$, $\u:= \v/\|\v\|$ is uniformly distributed on the unit sphere. Then can write:
\begin{align*}
    \P_{\u}\brk!{ |\u \dotp \nabla f(\x)| \leq \beta \gamma } 
    &=
    \P_{\v}\brk!{ \abs{\v \dotp \nabla f(\x)} \leq \beta \gamma \norm{\v}}
    \\
    &\leq
    \P_{\v}\brk!{ \abs{\v \dotp \nabla f(\x)} \leq 2\beta \gamma \sqrt{d\log (1/ \gamma')}} + \P_\v \brk{\norm{\v} \geq 2\sqrt{d\log(1/ \gamma')}}
    \\
    &\leq 
    \P_{\v}\brk!{ \abs{\v \dotp \nabla f(\x)} \leq 2\beta \gamma \sqrt{d\log( 1/ \gamma')}} + \gamma'
    ,
\end{align*}
where the final inequality is since $\P_\v \brk{\norm{\v}^2 \leq 2d\log(1/ \gamma')} \geq 1- \gamma'$ for any $ \gamma'$ (see Lemma~\ref{lem:chi}).
On the other hand, since $\v \dotp \nabla f(\x) \sim \cN(0,\|\nabla f(\x)\|^2)$, we have for any $\gamma>0$ that
\begin{align*}
    \Pr\brk*{ |\v \dotp \nabla f(\x)| \leq \gamma }
    \leq
    \frac{2\gamma}{\|\nabla f(\x)\| \sqrt{2\pi}}
    \leq
    \frac{\gamma}{\|\nabla f(\x)\|}
    .
\end{align*}
Combining the inequalities, we have that $\sign(f(\x+ \gamma \u) - f(\x-  \gamma \u)) = \sign(\u \dotp \nabla f(\x))$ except with probability at most
\begin{align*}
    \inf_{ \gamma'>0} \bigg\{ \gamma' + \frac{2\beta \gamma \sqrt{d\log( 1/ \gamma')}}{\|\nabla f(\x)\|}\bigg\} = \lambda \text{(say)}
   ,
\end{align*}
and further choosing $\gamma' = \frac{\beta \gamma \sqrt{d}}{\norm{\nabla f(\w)}}$, we get that $\lambda \leq \frac{\beta \gamma \sqrt{d}}{\norm{\nabla f(\x)}}\bigg( 1 + 2\sqrt{\log \frac{\norm{\nabla f(\x)}}{\sqrt d \beta \gamma}} \bigg)$. 
As for the claim about the expectation, note that for any vector $\b \in \cS_d(1)$,
\begin{align*}
\Big|\E_{\u}&[\sign(f(\x+ \gamma \u)-f(\x- \gamma \u)) \u\tr \b] - \E_{\u}[\sign(\nabla f(\x) \dotp \u) \u\tr \b] \Big|
\leq 
2\lambda
,
\end{align*}
since with probability $1-\lambda$ the two expectations are identical, and otherwise, they  differ by at most $2$.
\end{proof}

\begin{restatable}[]{lem}{chi} 
	\label{lem:chi}
	For $\v \sim \cN(\0_d,\cI_d)$ and any $\lambda > 0$, it holds that $\norm{\v}^2 \leq d+4\log(1/\lambda)$ with probability at least $1-\lambda.$
\end{restatable}

\begin{proof}
Let $X = \norm{\v}^2$. Then $X$ is distributed Chi-squared with $d$ degrees of freedom, and so its moment generating function is $\E[e^{zX}] = (1-2z)^{-d/2}$ for $z<1/2$.
Using Markov's inequality we have, for all $t>0$ and $0<z<1/2$,
\begin{align*}
    \Pr(X \geq t)
    =
    \Pr(e^{zX} \geq e^{zt})
    \leq
    e^{-zt}\E[e^{zX}]
    =
    e^{-zt} (1-2z)^{-d/2}
    \leq
    e^{-z(t-d)}
    .
\end{align*}
Choosing $z=1/4$ and $t = d+4\log(1/\lambda)$ makes the right-hand side smaller than $\lambda$, as we require.
\end{proof}

}

\end{document}